\documentclass[12pt, a4paper]{article}

\usepackage[a4paper]{geometry}
\usepackage{pgfplots}
\pgfplotsset{
    x tick style={color=black},
    y tick style={color=black}
}

\usepackage{url}
\usepackage{listings}
\usepackage{amssymb}
\usepackage{graphicx}
\usepackage{algcompatible}
\usepackage{algorithm}
\usepackage{url}
\usepackage{rotating}
\usepackage{booktabs}
\usepackage{subfig}
\usepackage{amsmath}
\usepackage{bbm}

\usepackage[thicklines]{cancel} 

\renewcommand{\labelenumi}{(\alph{enumi})}
\renewcommand\theenumi\labelenumi

\usepackage[english]{babel}

\usepackage[utf8]{inputenc}
\usepackage{xspace}
\usepackage{amsmath,amsthm,amssymb,mathtools}
\usepackage{lmodern}

\usepackage[algo2e,ruled,vlined,linesnumbered]{algorithm2e}

\usepackage{xcolor}
\usepackage{tikz}
\usepackage{graphicx}
\usepackage{soul}

\allowdisplaybreaks[4]
\newtheorem{theorem}{Theorem}
\newtheorem{lemma}[theorem]{Lemma}
\newtheorem{corollary}[theorem]{Corollary}
\newtheorem{definition}[theorem]{Definition}

\newcommand{\om}{\textsc{OneMax}\xspace}
\newcommand{\onemax}{\om}

\newcommand{\cocz}{\textsc{COCZ}\xspace}
\newcommand{\omm}{\textsc{OneMinMax}\xspace}
\newcommand{\oneminmax}{\omm}
\newcommand{\lotz}{\textsc{LOTZ}\xspace}
\newcommand{\ojzj}{\textsc{OneJumpZeroJump}\xspace}

\newcommand{\lfe}{\textsc{LargeFront$_{\eps}$}\xspace}

\newcommand{\R}{\ensuremath{\mathbb{R}}}

\newcommand{\N}{\ensuremath{\mathbb{N}}} 

\DeclareMathOperator{\cDis}{cDis}
\DeclareMathOperator{\mei}{MEI}
\DeclareMathOperator{\opt}{opt}
\DeclareMathOperator{\hv}{HV}

\newcommand{\NSGA}{\mbox{NSGA-II}\xspace}

\let\originalleft\left
\let\originalright\right
\renewcommand{\left}{\mathopen{}\mathclose\bgroup\originalleft}
\renewcommand{\right}{\aftergroup\egroup\originalright}

\usepackage{hyperref}

\date{}
\newcommand{\eps}{\varepsilon} 
\begin{document}
\sloppy

\title{Approximation Guarantees for the Non-Dominated Sorting Genetic Algorithm II (NSGA-II)}

\author{Weijie Zheng\\ 
         School of Computer Science and Technology\\
         International Research Institute for Artificial Intelligence\\
       Harbin Institute of Technology
\and Benjamin Doerr\thanks{Corresponding author.}\\ Laboratoire d'Informatique (LIX)\\ \'Ecole Polytechnique, CNRS\\ Institut Polytechnique de Paris\\ Palaiseau, France}

\maketitle

\begin{abstract}
Recent theoretical works have shown that the NSGA-II efficiently computes the full Pareto front when the population size is large enough. In this work, we study how well it approximates the Pareto front when the population size is smaller.

For the \omm benchmark, we point out situations in which the parents and offspring cover well the Pareto front, but the next population has large gaps on the Pareto front. Our mathematical proofs suggest as reason for this undesirable behavior that the NSGA-II in the selection stage computes the crowding distance once and then removes individuals with smallest crowding distance without considering that a removal increases the crowding distance of some individuals.

We then analyze two variants not prone to this problem. For the NSGA-II that updates the crowding distance after each removal (Kukkonen and Deb (2006)) and the steady-state NSGA-II (Nebro and Durillo (2009)), we prove that the gaps in the Pareto front are never more than a small constant factor larger than the theoretical minimum. This is the first mathematical work on the approximation ability of the NSGA-II and the first runtime analysis for the steady-state NSGA-II. Experiments also show the superior approximation ability of the two NSGA-II variants.
\end{abstract}

\section{Introduction}\label{sec:int}
While the theory of evolutionary algorithms (EAs), in particular, the mathematical runtime analysis, has made substantial progress in the last 25 years~\cite{NeumannW10,AugerD11,Jansen13,ZhouYQ19,DoerrN20}, the rigorous understanding of multi-objective EAs (MOEAs), started 20 years ago~\cite{LaumannsTZWD02,Giel03,LaumannsTZ04}, is less developed and is massively lagging behind their success in practice. However, in the last years some significant progress has been made, for example~\cite{BianQT18ijcaigeneral,RoostapourNNF19,QianYTYZ19,QianBF20,BianFQY20,ZhengD23ecj,Crawford21,DinotDHW23}. In particular, now the first analyses of MOEAs that are massively used in practice have appeared, namely for the MOEA/D~\cite{LiZZZ16}, the SMS-EMOA~\cite{BianZLQ23}, and most notably the \NSGA~\cite{ZhengLD22} (see~\cite{ZhengD23aij} for the journal version), the by far dominant algorithm in practice~\cite{ZhouQLZSZ11}.

The analysis of the \NSGA~\cite{ZhengLD22} proved that several variants of this algorithm can compute the full Pareto front of the \oneminmax and \lotz benchmarks efficiently when the population size is chosen by a constant factor larger than the size of the Pareto front (which is $n+1$ for these two problems in $\{0,1\}^n$). However, it was also proven (for the \oneminmax problem with problem size $n$) that a population size strictly larger than the Pareto front is necessary -- if these two sizes are only equal, then with probability $1 - \exp(-\Omega(n))$ for an exponential number of iterations the population of the \NSGA does not cover a constant fraction of the Pareto front. Experiments show that this fraction is roughly 20\% for the \oneminmax benchmark and roughly 40\% for the \lotz benchmark. Several runtime analyses of the \NSGA quickly followed this first work, see the literature review in Section~\ref{ssec:nsgaii}. However, all these works only discuss the efficiency of covering the full Pareto front.

Since we cannot always assume that the \NSGA is run with a population size larger than the Pareto front size by a constant factor -- both because the algorithm user does not know the size of the Pareto front and because some problems have a so large Pareto front that using a comparably large population size is not possible --, a deeper understanding of the approximation performance of the \NSGA is highly desirable. This is the target of this work. 

There is some reason to be optimistic: The experiments conducted in~\cite{ZhengLD22} for the case that the population size equals the size of the Pareto front not only gave the negative result that 20\% or 40\% of the Pareto front was not covered, but they also showed that the missing points are relatively evenly distributed over the Pareto front: the largest empty interval ever seen in all experiments was of length~$4$. Hence the population evolved by the \NSGA in these experiments was a very good approximation of the Pareto front.

\textbf{Our results:} Unfortunately, we observe that these positive findings do not extend to smaller population sizes. However, our negative results let us detect in the selection mechanism of the \NSGA a reason for the lower-than-expected approximation capability. This suggests a natural modification of the algorithm. For this, we prove that it computes populations that approximate the Pareto fronts of the \oneminmax and \lotz benchmarks very well. 

In detail, when we ran experiments for \omm with problem size $n=601$ and population sizes $N=(n+1)/2=301, \lceil (n+1)/4 \rceil=151$, and $\lceil (n+1)/8 \rceil=76$, empty intervals with sizes around $8$, $15$, and $26$ regularly occurred (see Table~\ref{tbl:omm}).
Unfortunately, due to the complicated population dynamics of the \NSGA, we were not able to prove a meaningful lower bound. Nevertheless, the experimental results show that even on a simple benchmark like \oneminmax, the largest empty interval the population has on the Pareto front is significantly larger than the optimal value $\lceil \frac{(n+1)-1}{N-1}\rceil$ ($3$, $5$, and $9$ for these population sizes above), which would result from a perfect distribution of the population on the Pareto front. 

To better understand how this discrepancy can arise, we regard two synthetic examples. We show that when the combined parent and offspring population is such that each point on the Pareto front is covered exactly once (this implies that the population size is essentially half the size of the Pareto front), then with high probability the next parent population does not cover an interval of length $\Theta(\log n)$ on the Pareto front (whereas simply removing every second point would give a population such that each point on the Pareto front has a neighbor that is covered by the population). We further construct a more artificial example where the combined parent and offspring population covers the Pareto front apart from isolated points, but the next parent population does not cover an interval of length $n/3$ of the Pareto front. 

The reason why we were able to construct such examples is the following property of the selection scheme of the \NSGA. To select the new parent population, the \NSGA uses as first criterion the non-dominated sorting and then the crowding distance. The crowding distance, however is not updated during the selection process. That is, while removing individuals with smallest crowding distance, the changing crowding distance of the remaining individuals is not taken into account, but instead the algorithm proceeds with the initial crowding distance. We assume that this design choice was made for reasons of efficiency -- by not updating the crowding distance, it suffices to sort the combined parent and offspring population once and then remove the desired number of individuals. We note that this shortcoming of the traditional selection method of the \NSGA was, with intuitive arguments, already detected by Kukkonen and Deb~\cite[Figure~2]{KukkonenD06}, a paper which unfortunately is not too well-known in the community (we overlooked it when preparing the conference version~\cite{ZhengD22gecco} and none of the reviewers detected this oversight; our deepest thanks to Hisao Ishibuchi for pointing us to the work of Kukkonen and Deb). 

There is a natural remedy to this shortcoming, proposed also by Kukkonen and Deb~\cite{KukkonenD06}, and this is to sequentially remove individuals always based on the current crowding distance. This procedure can be implemented very efficiently: The removal of one individual changes the crowding distance of at most $4$ other individuals (in a bi-objective problem), so at most $4$ crowding distance values need to be updated. There is no need for a new sorting from scratch when we use as a data structure a priority queue. With this implementation, the selection based on the current crowding distance takes not more than $O(N \log N)$ operations, which is the same asymptotic complexity as the one of sorting the individuals by their initial crowding distance in the original \NSGA. We note that both operations are fast compared to the non-dominated sorting step with its quadratic time complexity.

For this modified \NSGA, the problems shown above for the traditional \NSGA cannot occur. For problem size $n=601$, the modified algorithm for $N = 301$ never created an empty interval on the Pareto front larger than optimal value of $3$. For $N=151$, in more than half the iterations all empty intervals observed the optimal value of $5$, in the other iterations the maximum empty interval (MEI) had a length of~$6$. For $N=76$, the median MEI value was $11$ (optimal value:~$9$). Hence the modified algorithm distributes the population on the Pareto front in a significantly more balanced manner.

For this algorithm, we can also prove a guarantee on the approximation quality. After a time comparable to the time needed to find the two extremal solutions of the front, for all future generations with probability one the largest empty interval on the Pareto front has length at most $\max\{\frac{2n}{N-3},1\}$, hence at most a constant factor larger than the theoretical minimum of $\lceil \frac{(n+1)-1}{N-1} \rceil$. Consequently, even with a population size not large enough to cover the full Pareto front, this algorithm computes very good approximations to the Pareto front. 

There is a second variant of the \NSGA proposed in the past that, by definition, is not prone to the problem of working with initial crowding distance values, namely the steady-state \NSGA proposed by Durillo et al.~\cite{DurilloNLA09}, which generates a single offspring per iteration. There are no theoretical result on this algorithm yet, but the empirical results of Nebro and Durillo showed its very competitive approximation strength. For this reason, we also conduct a mathematical runtime analysis of this algorithm on the \omm benchmark. We prove the same good approximation guarantees as for the \NSGA with current crowding distance. Our experiments in Section~\ref{sec:exp} verify this similarity and also indicate a more steady behavior of the steady-state \NSGA.

This work extends our conference paper~\cite{ZhengD22gecco} majorly in the following ways. This version obtains tighter estimates of the relation between the $\eps$-dominance and maximal empty interval length. We further discuss the relation between the hypervolume indicator and the maximum empty interval length, which was not contained in the conference version. This version improves approximation guarantee for the \NSGA with current crowding distance from $4n/(N-3)$ to $2n/(N-3)$. We also added a new section discussing the approximation guarantee of the steady-state \NSGA and the corresponding experiments. Besides, this version contains all mathematical proofs that were omitted in the conference version for reasons of space.

This work is organized as follows. Section~\ref{sec:pre} brief{}ly introduces bi-objective optimization and the \NSGA. Section~\ref{sec:app} discusses the approximation measures that will be used in this work. The approximation difficulties of the traditional \NSGA are theoretically shown via two synthetic examples in Section~\ref{sec:nsgaii}. Section~\ref{sec:onthefly} introduces our modified variant of the \NSGA and conducts the theoretical analysis of its approximation ability. Our experiments are discussed in Section~\ref{sec:exp}. Section~\ref{sec:con} concludes this work.

\section{Preliminaries}\label{sec:pre}

\subsection{Bi-objective Optimization and the \omm Benchmark}

In this paper, we regard on bi-objective optimization problems $f=(f_1,f_2):\{0,1\}^n \rightarrow \R^2$ with each objective to be maximized. For $x,y\in\{0,1\}^n$, we say that $x$ \emph{strictly dominates} $y$, denoted by $x\succ y$, if $f_1(x)\ge f_1(y), f_2(x)\ge f_2(y)$ and at least one of the inequalities is strict. If $x$ cannot be strictly dominated by any solution in $\{0,1\}^n$, we say that $x$ is \emph{Pareto optimal} and that $f(x)$ is a \emph{Pareto front point}. The set of all Pareto front points is called \emph{Pareto front}. The typical aim for a multi-objective optimizer is to compute the Pareto front, that is, compute a set of solutions such that $f(P)$ is the Pareto front, or to approximate it well.

We shall work with the popular bi-objective benchmark \omm. \omm was first proposed in~\cite{GielL10} and is similar to the \cocz benchmark defined in~\cite{LaumannsTZ04}. The first objective of \omm counts the number of zeros in the bit-string, and the second objective counts the number of ones. More specifically, for any $x=(x_1,\dots,x_n)\in\{0,1\}^n$, the \omm function is defined by 
$$f(x)=(f_1(x),f_2(x))=\left(n-\sum_{i=1}^n x_i, \sum_{i=1}^n x_i\right).$$
It is not difficult to see that any solution $x\in\{0,1\}^n$ is Pareto optimal and that the Pareto front is $M:=\{(0,n),(1,n-1),\dots,(n,0)\}$.

\subsection{The Non-Dominated Sorting Genetic Algorithm II (\NSGA)}\label{ssec:nsgaii}
We now give a brief introduction to the \NSGA, which was first proposed in~\cite{DebPAM02} and now is the by far dominant MOEA in practice~\cite{ZhouQLZSZ11}. It works with a fixed population size $N$. Consequently, each time new individuals are generated, the \NSGA needs to remove individuals to maintain this fixed population size. To this aim, the \NSGA computes a complete order on the combined parent and offspring population that uses the dominance as the first criterion and a diversity measure (crowding distance) as the second criterion, and removes the worst individuals. 

In more detail, after the random initialization, in each generation $t$, an offspring population $Q_t$ with size $N$ is generated from the parent population $P_t$. The \NSGA now needs to remove $N$ individuals from the combined population $R_t=P_t\cup Q_t$. To this aim, it divides $R_t$ into several fronts $F_1,F_2,\dots,$ where $F_1$ is the set of the non-dominated solutions in $R_t$, and $F_i,i>1,$ is the set of the non-dominated solutions in $R_t\setminus \{F_1,\dots,F_{i-1}\}$. For the first index $i^*$ such that the size of $\bigcup_{i=1}^{i^*}F_i$ is at least $N$, the \NSGA will calculate the \emph{crowding distance}, denoted by $\cDis$, of the individuals in $F_{i^*}$ as follows.\footnote{If the crowding distance is also used for the parent selection (like in tournament selection), then now (Algorithm~\ref{alg:nsgaii}, step~\ref{ste:cDis}) also the crowding distance of the individuals in $F_{1},\dots,F_{i^*-1}$ is computed.} For each objective, the individuals are sorted according to their objective values. The $\cDis$ value of the first and last point in the sorted list is infinite. For the other individuals, the $\cDis$ with respect to the current objective is the normalized distance of the objective values of its two neighbors in the list. The complete $\cDis$ of an individual is the sum of its $\cDis$ components for all objectives. See Algorithm~\ref{alg:cDis} for a complete description of the computation of the crowding distance. Then the $|\bigcup_{i=1}^{i^*}F_i|-N$ individuals in $F_{i*}$ with smallest crowding distance value and all individuals in $F_i, i>i^*$, are removed (ties broken randomly). The complete \NSGA framework is shown in Algorithm~\ref{alg:nsgaii}. 

We note that, for reasons of generality, in Algorithm~\ref{alg:cDis} we make no assumption on how individuals with equal objective values are sorted. As pointed out in~\cite{BianQ22}, for bi-objective problems the sorting w.r.t.\ the second objective can be taken as the inverse of the first sorting. This has mild algorithmic advantages (but note that the quadratic time complexity of the non-dominated sorting procedure dominates the complexity of the selection stage) and can lower the minimum required population size to compute the whole Pareto front by a factor of two. It is clear that this idea can only make a difference when two individuals with identical value in one objective are present, as only then the sorting is not unique. In our setting with the population size smaller than the size of the Pareto front, we do not expect this to happen too often, so we do not expect significantly better results under this assumption.

Since it is not the most central aspect of our work, we have not yet discussed how the offspring population is computed. As in~\cite{ZhengLD22}, we shall regard a mutation-only version of the \NSGA. For a problem like \oneminmax, composed of two very simple unimodal objectives, we do not expect that the use of crossover gives significant advantages. Also, we are convinced that our proofs can easily be extended to variants that use crossover with some constant probability (less than one) -- we note that in particular the central result of this work regarding the selection of the next population (Lemma~\ref{lem:selection}) does not rely on any assumption about how the offspring are created. As mutation operators, we shall regard the two classic ones of \emph{one-bit mutation}, which flips a random bit in the argument, and \emph{standard bit-wise mutation}, which flips each bit independently with probability~$\frac 1n$. We consider three ways to select the parents for the mutation (mating selection). In \emph{fair selection}, we let each of the $N$ parents generate exactly one offspring. In \emph{random selection}, we $N$ times independently and uniformly at random (hence ``with replacement'') select a parent and let it create an offspring. In \emph{binary tournament selection}, we $N$ times independently and uniformly at random pick two parents, and let the better one (according to non-dominated sorting and crowding distance, ties broken randomly) create an offspring.

\begin{algorithm}[tb]
    \caption{Computation of the crowding distance $\cDis(S)$}
    \textbf{Input:} $S=\{S_1,\dots,S_{|S|}\}$, a set of individuals\\
    \textbf{Output:} $\cDis(S)=(\cDis(S_1),\dots,\cDis(S_{|S|}))$, where $\cDis(S_i)$ is the crowding distance for $S_i$
		
    \begin{algorithmic}[1]
    \STATE $\cDis(S)=(0,\dots,0)$
    \FOR {each objective $f_i$}
    \STATE {Sort $S$ in order of descending $f_i$ value: $S_{i.1},\dots,S_{i.{|S|}}$}
    \STATE {$\cDis(S_{i.1})=+\infty, \cDis(S_{i.{|S|}})=+\infty$}
    \FOR {$j=2,\dots, |S|-1$}
    \STATE {$\cDis(S_{i.j})=\cDis(S_{i.j}) + \frac{f_i(S_{i.{j-1}})-f_i(S_{i.{j+1}})}{f_i(S_{i.1})-f_i(S_{i.{|S|}})}$}
    \ENDFOR
    \ENDFOR
    \end{algorithmic}
    \label{alg:cDis}
\end{algorithm}

\begin{algorithm}[!ht]
    \caption{\NSGA}
    \begin{algorithmic}[1]
    \STATE {Uniformly at random generate the initial population $P_0=\{x_1,x_2,\dots,x_N\}$ with $x_i\in\{0,1\}^n,i=1,2,\dots,N.$}\label{ste:initialize}
    \FOR{$t = 0, 1, 2, \dots$} \label{ste:iterate}
    \STATE {Generate the offspring population $Q_t$ with size $N$}\label{ste:generate}
    \STATE {Use fast-non-dominated-sort() in~\cite{DebPAM02} 
    to divide $R_t$ into fronts $F_1,F_2,\dots$}
    \label{ste:sort}
    \STATE {Find $i^* \ge 1$ such that $|\bigcup_{i=1}^{i^*-1}F_i| < N$ and $|\bigcup_{i=1}^{i^*}F_i| \ge N$ }\label{ste:rank}
    \STATE {Use Algorithm~\ref{alg:cDis} to separately calculate the crowding distance of each individual in $F_{1},\dots,F_{i^*}$}\label{ste:cDis}
    \STATE {Let $\tilde{F}_{i^{*}}$ be the $N-|\bigcup_{i=1}^{i^*-1}F_{i}|$ individuals in $F_{i^*}$ with largest crowding distance, chosen at random in case of a tie}\label{ste:final front}
    \STATE {$P_{t+1}=\left(\bigcup_{i=1}^{i^*-1}F_i\right)\cup\tilde{F}_{i^*}$}\label{ste:new parents}
    \ENDFOR 
    \end{algorithmic}
    \label{alg:nsgaii}
\end{algorithm}


We brief{}ly review the state of the art in runtime analysis for the \NSGA. In the first mathematical runtime analysis\footnote{As usual, we call a mathematical runtime analysis a theoretical work estimating the number of iterations or fitness evaluations it takes to reach a certain goal. This is different from the implementational complexity of the operations for each iteration, as discussed, e.g., already in the original \NSGA paper~\cite{DebPAM02}.} of the \NSGA~\cite{ZhengLD22}, it was shown that the \NSGA with several mating selection and mutation strategies and population size $N = \Omega(n)$ sufficiently large, efficiently computes the Pareto fronts of the \oneminmax and \lotz benchmarks, namely in expected $O(Nn \log n)$ and $O(N n^2)$ fitness evaluations. For $N = \Theta(n)$, these are the same asymptotic complexities as those known for the basic \emph{global simple evolutionary multi-objective optimizer (GSEMO)}~\cite{Giel03}, namely $O(n^2\log n)$ for \omm and $O(n^3)$ for \lotz. However, the work~\cite{ZhengLD22} also showed that with a population size of $n+1$, that is, equal to the size of the Pareto front, it takes at least an exponential time to compute a population covering the Pareto front better than with a constant-factor loss.

This work has quickly led to several follow-up works beyond the conference version~\cite{ZhengD22gecco} of this work. In~\cite{BianQ22}, it was observed that using the same sorting for the two objectives allows to lower the minimum required population size to twice the size of the Pareto front (this is proven for the \lotz benchmark, but the argument can easily be extended to \oneminmax). Also, in this work for the first time an \NSGA with crossover is regarded and runtime guarantees analogous to those in~\cite{ZhengLD22} are proven. The most profound result of this work is that significant runtime improvements (from $O(N n^2)$ to $O(N n)$ when optimizing \lotz with $N = \Theta(n)$) can be obtained from a newly proposed tournament selection operator, which chooses the tournament size chosen uniformly at random from the range $[1..N]$. 
This is particularly interesting in that here a uniform random choice was successful, whereas most previous works using random parameter choices employ heavy-tailed distributions, see, e.g.,~\cite{DoerrLMN17,AntipovBD21gecco,ZhengD23ecj,DangELQ22}. The first mathematical runtime analysis on a multimodal problem, the \ojzj benchmark from~\cite{ZhengD23ecj}, was conducted in~\cite{DoerrQ23tec}. It shows that when the population size is at least $4$ times the Pareto front size, then the \NSGA with the right population size is as effective as the GSEMO on this benchmark. The first lower bounds for the \NSGA~\cite{DoerrQ23LB} show that the previous results for \oneminmax~\cite{ZhengLD22} and \ojzj~\cite{DoerrQ23tec} are asymptotically tight, even for larger population sizes. This implies that increasing the population size above the minimum required size immediately leads to asymptotic performance losses. This is very different from single-objective optimization, where often a larger range of population sizes gives the same asymptotic performance, see, e.g., \cite{JansenJW05,DoerrK15} for such results for the \onemax benchmark. Two results showing that crossover can give asymptotic performance gains appeared in parallel~\cite{DangOSS23aaai,DoerrQ23crossover}. The first runtime analysis for a combinatorial optimization problem, the bi-objective minimum spanning tree problem previously analyzed for the GSEMO~\cite{Neumann07}, appeared in~\cite{CerfDHKW23}. A runtime analysis in the presence of noise (together with the independent parallel work~\cite{DinotDHW23} the first mathematical runtime analysis of a MOEA in a noisy environment) was conducted in~\cite{DangOSS23gecco}. That the \NSGA can have difficulties with more than two objectives was shown for \oneminmax in~\cite{ZhengD23tevc}, whereas the {NSGA-III} was proven to be efficient on the $3$-objective \oneminmax problem in~\cite{WiethegerD23}.

As pointed in Section~\ref{sec:int}, all these theoretical works consider the time taken to cover the full Pareto front, so no other work exists on approximating the Pareto front.



\section{Approximation Measures}\label{sec:app}

For reasons of efficiency, instead of computing the whole Pareto front, one often resorts to approximating it, that is, to computing a set of solutions that is a reasonable representation for the whole front. This raises the question of how to measure the approximation quality of such a set of solutions. 
Several different approximation measures have been proposed in the literature such as multiplicative $\eps$-dominance~\cite{LaumannsTDZ02}, various kinds of generational distances~\cite{VanVeldhuizenL98,BosmanT03,CoelloR04}, or the hypervolume~\cite{ZitzlerT98}. For the \oneminmax problem regarded in this work, the particular structure of the problem suggests to use a more elementary measure, namely the maximum length of an empty interval on the Pareto front. We define this measure in Section~\ref{sec:mei}. We will then in Sections~\ref{sec:eps} and~\ref{sec:hv} compare the new measure with two classic measures, namely $\eps$-dominance, which has frequently been regarded in theoretical works, and hypervolume, which is mostly widely used in general~\cite{ShangIHP21}. We are optimistic that our new measure can equally easily be compared  to other measures.

\subsection{Maximal Empty Interval Size}\label{sec:mei}
For \omm with problem size $n$ that this paper will analyze, each possible objective value is on the Pareto front and the first objective values of full Pareto front are exactly $0,1,\dots,n$. Any missing Pareto front point can be directly seen in $[0..n]$, hence, we now simply use a measure about the size of the maximal empty interval, denoted as MEI, inside $[0..n]$ in terms of the solutions that one MOEA reaches and with respect to $f_1$ values. If the maximal empty interval size is as small as possible, then the MOEA can approximate the Pareto front as well as possible. 
The formal definition of the MEI of a set $U$ in the objective space is as follows.

\begin{definition}
Let $S=\{(s_1,n-s_1),\dots,(s_m,n-s_m)\}$ be a subset of the Pareto front $M$ of \omm. Let $j_1,j_2,\dots,j_{m}$ be the sorted list of $s_1,\dots,s_m$ in the increasing order (ties broken uniformly at random). We define the \emph{maximal empty interval size} of $S$, denoted by $\mei(S)$, as 
\begin{align*}
\mei(S)=\max\{j_{i+1}-j_i \mid i=1,\dots,m-1\}.
\end{align*}
For $n \in \N_{\ge 2}$, we further define 
\begin{align*}
\mei_{\opt}(N):=\min\{\mei(S) \mid S\subseteq M, |S|\le N, (0,n) \in S, (n,0) \in S\}.
\end{align*} 
Obviously, this is the smallest $\mei$ that an MOEA with a fixed population size $N$ can obtain when the extremal points $(0,n)$ and $(n,0)$ are covered.
\end{definition}

It is not difficult to see that $\mei_{\opt}(N)$ is witnessed by a set $S$ as evenly distributed as possible. We explicitly formulate this observation in the following lemma.

\begin{lemma}
For all $N\in \N_{\ge 2}$, we have $\mei_{\opt}(N)=\lceil \frac{n}{N-1}\rceil$.
\label{lem:optimalmei}
\end{lemma}

\begin{proof}
Consider any $S\subseteq M$ with $|S|\le N, (0,n)\in S,$ and $(n,0) \in S$. Let $J$ be the set of the first objective values of $S$, and let $J=\{j_1,\dots,j_{|J|}\}$ with $j_a< j_b$ for $a<b$. Then we have $j_1=0$ and $j_{|J|}=n$, and thus
\begin{align*}
\sum_{i=1}^{|J|-1}j_{i+1}-j_i = j_{|J|}-j_1=n.
\end{align*} 
Hence, there exists $i_0\in[1..|J|-1]$ such that $j_{i_0+1}-j_{i_0}\ge n/(|J|-1)$. Since $j_{i_0+1}$ and $j_{i_0}$ are integers, we then have $j_{i_0+1}-j_{i_0}\ge \lceil n/(|J|-1) \rceil$. As $\mei(S)=\max\{j_{i+1}-j_{i}\mid i=1,\dots,|J|\}$, we know that $\mei(S)\ge  \lceil n/(|J|-1) \rceil \ge  \lceil n/(N-1) \rceil$,
where the last inequality uses $|J| \le |S| \le N$. Hence, $\mei_{\opt}(N) \ge \lceil {n}/{(N-1)}\rceil$.

Now we prove that $\mei_{\opt}(N) \le \lceil n/(N-1) \rceil$. Let $j_i=\min\{(i-1)\left\lceil n/(N-1)\right\rceil,n\}$ for $i=1,\dots, N$. Then $j_1=0$ and $j_N=n$. Consider the set $S' = \{(j_i,n-j_j) \mid i \in [1..N]\}$. 
It is not difficult to see that
\begin{align*}
j_{i+1} - j_i = \min\left\{(i+1) \left\lceil\tfrac{n}{N-1}\right\rceil,n\right\}-\min\left\{i\left\lceil\tfrac{n}{N-1}\right\rceil,n\right\}\le \left\lceil\tfrac{n}{N-1}\right\rceil
\end{align*}
for all $i \in [1..N-1]$. Hence $\mei(S')\le\lceil n/(N-1) \rceil$. By definition of $\mei_{\opt}(N)$, we know that $\mei_{\opt}(N) \le \mei(S') \le \lceil n/(N-1) \rceil$. 
\end{proof}

\subsection{$\eps$-Dominance}\label{sec:eps}
This subsection discusses the optimal approximation quality w.r.t.\ the classic $\eps$-dominance measure for \omm and compares this measure with the MEI.
\subsubsection{Background and Definition}

One way to measure the quality of solution sets is via $\eps$-dominance, which is a relaxed notion of dominance first defined in~\cite{LaumannsTDZ02}. A set $S$ of objective vectors is then called an $\eps$-approximation of the Pareto front if each point of the Pareto front is $\eps$-dominated by a point from~$S$. 

In this subsection, we only discuss multiplicative $\eps$-dominance and simply call it $\eps$-dominance, as it is the main variant discussed in~\cite{LaumannsTDZ02} and the variant most used in follow-up works. 
Here is the formal definition. 

\begin{definition}[(Multiplicative) $\eps$-dominance \cite{LaumannsTDZ02}]
Let $\eps>0$ and $m>0$ be the number of objectives. For $u,v \in \R^m$, we say $u$ \emph{$\eps$-dominates $v$}, denoted by $u \succeq_\eps v$, if and only if $(1+\eps)u \ge v$, that is, $(1+\eps)u_i \ge v_i$ for all $i=1,\dots,m$. 
	
Let $W=\{u \mid u \in \R^m\}$ be the whole objective vector set for a given problem. We say a subset $S\subseteq W$ is an \emph{$\eps$-approximation} for this problem if and only if for each $v \in W$, there exists $u\in S$ such that $u \succeq_\eps v$.
\label{def:eps}
\end{definition}

Here we briefly review the theoretical works utilizing $\eps$-dominance relation as the measure to evaluate the approximation performance of the MOEAs or as the basis to design the MOEAs considering more diversity in the survival selection. Horoba and Neumann~\cite{HorobaN08} defined the \lfe function and proved that to obtain an $\eps$-approximation, the GSEMO needs a $2^{\Omega(n^{1/4})}$ runtime with $1-2^{-\Omega(n^{1/4})}$ probability, while the GSEMO with a diversity strategy based on $\eps$-dominance only requires an expected runtime of $O(n^2\log n)$. A similar work with respect to the additive $\eps$-dominance relation was conducted in~\cite{HorobaN09}.
To reach an $\eps$-approximation for the \lfe, Brockhoff, Friedrich, and Neumann~\cite{BrockhoffFN08} proved that the runtime for the $(\mu+1)$-simple indicator-based evolutionary algorithm ($(\mu+1)$-SIBEA) is $O(n^2\log n)$. 

For the GSEMO with $\eps$-dominance diversity strategy, Neumann and Reichel~\cite{NeumannR08} proved that it can achieve good approximations {for the minimum cut problem} in expected polynomial time without restrictions on the graph weight, while the original GSEMO requires the bounded weight to ensure the efficient approximation. The efficiency of the GSEMO with $\eps$-dominance diversity strategy can also be witnessed in~\cite{NeumannRS11,PourhassanSN19}. 

Gutjahr~\cite{Gutjahr12} replaced the survival selection of the SEMO considering the common dominance by the additive $\eps$-dominance, inserted the modified SEMO into the adaptive Pareto sampling (APS) framework, and proved that for one stochastic multi-objective combinatorial optimization problem, the expected runtime for such APS can be bounded from above, with the bound depending on the expected runtime of the original SEMO on the deterministic counterpart.

\subsubsection{Relation of Maximal Empty Intervals and $\eps$-Approximations}

The two approximation measures MEI and $\eps$-approximation are almost equivalent. In this subsection, we show that for any $S \subseteq M$ with $(0,n),(n,0)\in S$ the smallest $\eps$ rendering $S$ an $\eps$-approximation satisfies $(\mei(S)-1)/(n-\mei(S))\le \eps \le \mei(S)/(n-\mei(S))$. 

 We first prove the following relation between the MEI and $\eps$-dominance.

\begin{lemma}\label{lem:eps}
Let $S$ be a subset of the Pareto front $M$ of \omm with $(0,n),(n,0)\in S$. Let $\eps = \mei(S)/(n-\mei(S))$. Then $S$ is an $\eps$-approximation for $\overline{M}:=\{(x,n-x)\mid x\in [0,n]\}$ (and thus also for $M$).
\end{lemma}
\begin{proof}
Let $J$ be the set of the first objective values of $S$. More specifically, let $J=\{j_1,\dots,j_{|J|}\}$ with $j_a<j_b$ for $a < b$. From the assumption, we know $j_1=0$ and $j_{|J|}=n$. Consider any $a\in[1..|J|-1]$ and let $v_1=(j_a,n-j_a)$ and $v_2=(j_{a+1},n-j_{a+1})$. For any $v=(r,n-r)$ with $r \in [j_a,j_{a+1}]$, we have 
\begin{align*}
(1+\eps)(n-(j_{a+1}-j_a))&={}\left(1+\frac{\mei(S)}{n-\mei(S)}\right)(n-(j_{a+1}-j_a))\\
&={}\frac{n(n-(j_{a+1}-j_a))}{n-\mei(S)} \ge n.
\end{align*}
Hence we know that $(1+\eps)j_a\ge r$ or $(1+\eps)(n-j_{a+1})\ge n-r$. Thus $v$ is $\eps$-dominated by one of $v_1$ and $v_2$. 
\end{proof}

Similar to $\mei_{\opt}$, we define an optimal value for $\eps$. From Definition~\ref{def:eps}, it is not difficult to see that the smaller $\eps$, the better $S$ approximates the set $W$. 
Let $A(M,\eps)$ be the collection of all $S$ that are an $\eps$-approximation of $M$. Then, for $N\in\N_{\ge2}$, let 
\begin{align*}
\eps_{\opt}(N)=\inf\{\eps > 0 \mid \exists S\in A(M,\eps), |S|=N, (0,n)\in S, (n,0)\in S\}.
\end{align*}
Obviously, this is the smallest $\eps$ so that a population of size $N$ covering the two extremal points can be an $\eps$ approximation of~$M$.
 
From Lemma~\ref{lem:eps}, we easily obtain the following upper bound of the optimal $\eps$. 
\begin{corollary}\label{cor:epsupper}
For all $N\in\N_{\ge2}$, we have $\eps_{\opt}(N)\le \mei_{\opt}(N)/(n-\mei_{\opt}(N))$.
\end{corollary}

Now we establish a lower bound for $\eps$ w.r.t. $\mei$, and also a lower bound for the optimal $\eps$ as its corollary.
\begin{lemma}\label{lem:epslow}
Let $S$ with $(0,n),(n,0)\in S$ be an $\eps$-approximation for ${M}$. Then $\eps \ge  (\mei(S)-1)/(n-\mei(S))$.
\end{lemma}
\begin{proof}
Let $v_1=(i_1,n-i_1) \in S$ and $v_2=(i_2,n-i_2)\in S, i_1<i_2$ be any two neighboring points in $S$. Let $v_3=(i_3,n-i_3)$ with $i_3\in[i_1+1..i_2-1]$. If there exists an $i'\ge i_2$ such that $v'=(i',n-i') \succeq_{\eps} v_3$, then $(1+\eps)(n-i') \ge n-i_3$. Since $i'\ge i_2$, we have $n-i_2 \ge n-i'$ and thus $(1+\eps)(n-i_2) \ge n-i_3$. Together with $i_2 > i_3$, we know $v_2 \succeq_{\eps} v_3$. Analogously, if there exists an $\eps$ and $i'\le i_1$ such that $v'=(i',n-i')\succeq_{\eps} v_3$, then $v_1\succeq_{\eps} v_3$. Hence one of $v_1$ and $v_2$ $\eps$-dominates $v_3$.

If $v_2\succeq_\eps v_3$, then $(n-i_2)(1+\eps) \ge n-i_3$, that is, $\eps \ge (n-i_3)/(n-i_2)-1$. If $v_1\succeq_\eps v_3$, then $(1+\eps)i_1 \ge i_3$, that is, $\eps \ge i_3/i_1-1$. Hence, 
\begin{align}
\eps \ge \min\left\{\frac{n-i_3}{n-i_2}-1,\frac{i_3}{i_1}-1\right\}=\min\left\{\frac{i_2-i_3}{n-i_2},\frac{i_3-i_1}{i_1}\right\}.
\label{eq:minv}
\end{align}
Let $a=i_1n/(n-(i_2-i_1))$. While not important for the proof, we note that this is the value for $i_3$ that renders the two arguments of the minimum equal and thus maximizes~\eqref{eq:minv}.
If $a$ is an integer, then from taking $i_3=a$ we obtain $\eps \ge (i_2-i_1)/(n-(i_2-i_1))$.
If $a$ is not an integer, let $a' := \lfloor a \rfloor$ and $\delta:=a-a'$. Exploiting~\eqref{eq:minv} with $i_3=a'$ and $i_3=a'+1$, we obtain
\begin{align*}
\eps &\ge{} \max\left\{\min\left\{\frac{i_2-a'}{n-i_2},\frac{a'-i_1}{i_1}\right\},\min\left\{\frac{i_2-(a'+1)}{n-i_2},\frac{a'+1-i_1}{i_1}\right\}\right\}\\
&={}\max\left\{\frac{a'-i_1}{i_1},\frac{i_2-(a'+1)}{n-i_2}\right\}\\
&={} \max\left\{\frac{i_2-i_1}{n-(i_2-i_1)}-\frac{\delta}{i_1},\frac{i_2-i_1}{n-(i_2-i_1)}-\frac{1-\delta}{n-i_2}\right\}\\
&={} \frac{i_2-i_1}{n-(i_2-i_1)} - \min\left\{ \frac{\delta}{i_1}, \frac{1-\delta}{n-i_2} \right\}\ge \frac{i_2-i_1-1}{n-(i_2-i_1)},
\end{align*}
where the last inequality uses that the maximal value of $\min\{\delta/i_1, (1-\delta)/(n-i_2)\}$ is $1/(n-(i_2-i_1))$ when $\delta=i_1/(n-(i_2-i_1))$.

Let $J=\{j_1,\dots,j_{|J|}\}$ be the set of first objective values in $S$ and ordered such that $j_b<j_c$ for all $b < c$. We have just shown that $\eps \ge \frac{j_{b+1}-j_b -1}{n-(j_{b+1}-j_b)}$ for all $b\in[1..|S|-1]$. Hence
\begin{align}
\eps \ge \max_{b\in[1..|J|-1]}\left\{\frac{j_{b+1}-j_b -1}{n-(j_{b+1}-j_b)}\right\}=\frac{\mei(S)-1}{n-\mei(S)},
\label{eq:mineps}
\end{align}
where we use the fact that $(j_{b+1}-j_b -1)/(n-(j_{b+1}-j_b))$ increases as $j_{b+1}-j_b$ increases.
\end{proof}

\begin{corollary}\label{cor:epslow}
For all $N\in\N_{\ge2}$, we have $\eps_{\opt}(N)\ge (\mei_{\opt}(N)-1)/(n-\mei_{\opt}(N))$.
\end{corollary}

\subsection{Hypervolume}\label{sec:hv}

In this subsection, we discuss the classic hypervolume indicator, the most common measure for how well a set of solutions approximates the Pareto front, and its relation to the MEI. 

\subsubsection{Background and Definition}

The hypervolume indicator, denoted by HV in the remainder, was first proposed by Zitzler and Thiele~\cite{ZitzlerT98}. It is the most intensively used measure for approximation quality in evolutionary multi-objective optimizations~\cite{ShangIHP21}. Here ``hypervolume'' refers to the size of the space covered by the solution set (with respect to a reference point). Different from $\mei$, $\eps$-dominance, and generational distance measures~\cite{VanVeldhuizenL98,BosmanT03,CoelloR04}, the HV does not require a prior knowledge or estimate about the Pareto front, which is a huge advantage in practical applications.

\begin{definition}[Hypervolume (HV)]
Let $m>0$ be the number of objectives, and let $W=\{u \mid u \in \R^m\}$ be the whole objective vector set for a given problem. Let $r\in \R^m$ be such that $u \succeq r$ for all $u\in W$. Let $\mathcal{L}$ denote the Lebesgue measure in~$\R^m$. Then the \emph{hypervolume} of a subset $J\subseteq W$ is defined as 
\[
\hv(J,r)=\mathcal{L}\left(\bigcup_{u\in J} H_{u,r} \right),
\] 
where $H_{u,r}=\prod_{i=1}^m[r_i,u_i]$ is the hyper-rectangle defined by the corners $u$ and $r$.
\label{def:hv}
\end{definition}

In addition to its wide usage in practice, there are also some theoretical works discussing the HV. Brockhoff, Friedrich, and Neumann~\cite{BrockhoffFN08} proposed the \emph{$(\mu+1)$-simple indicator-based evolutionary algorithm}, $(\mu+1)$-SIBEA, and proved its efficient expected runtime of $O(n^2\log n)$ for reaching an $\eps$-approximation of the Pareto front of \lfe, for any $\eps \in \R_{>0}$. In contrast, the GSEMO with high probability needs at least exponential time to achieve an $\eps$-approximation~\cite{HorobaN08}. Besides this result,~\cite{BrockhoffFN08} also proved that the $(\mu+1)$-SIBEA with $\mu \ge n+1$ computes the full Pareto front of the \lotz function in an expected runtime of $O(\mu n^2)$. 

Nguyen, Sutton, and Neumann~\cite{NguyenSN15} proved that the expected runtime of the $(\mu+1)$-SIBEA, $\mu \ge n+1$, to compute the full Pareto front (hence to reach the maximum HV value possible with this population size) for the \omm problem is $O(\mu n \log n)$. They also proved that starting from some initial population and using a population size of only $\mu=\sqrt n+1$, the $(\mu+1)$-SIBEA with high probability takes at least exponential time to reach the maximum HV value (attainable with this population size). For the \lotz problem, they showed the following results. If $\mu \ge n+1$, then the algorithm can reach the maximum HV value (and thus compute the full Pareto front) in expected time $O(\mu n^2)$. If $1< \mu < n/3$, the expected time to reach the maximum hypervolume (and thus approximate the Pareto front best in this measure) is $O(\mu^3 n^3)$. 

While the work just discussed has shown that the $(\mu+1)$-SIBEA with small population sizes has difficulties to compute the optimal approximation of the Pareto front of \omm, this algorithm can easily compute good approximations. Doerr, Gao, and Neumann~\cite{DoerrGN16} proved that $(\mu+1)$-SIBEA with $\mu\le \sqrt n$ in expected time  $O(\mu n\log n)$ reaches a population with HV below the value for the full Pareto front by only an additive term of at most $2n^2/\mu$. 

We note that in addition to the above runtime analyses, other theoretical works on structural aspects of the HV exist,
e.g., \cite{AugerBB10,AugerBBZ12,BringmannF13ai,FriedrichNT15}. We will not discuss them here in full detail as our main focus in this work is the analysis of runtimes of evolutionary algorithms. 

However, we note that for discrete problems, the optimal HV value for a population with size less than the finite Pareto front size is not well understood. Emmerich, Deutz, Beume~\cite[Section~3]{EmmerichDB07} proved that the HV value for approximating the continuous Pareto front $\{(x,1-x)\mid x\in[0,1]\}$ with $\mu$ points plus the two extremal points $(1,0),(0,1)$ is maximized if the $\mu$ points are evenly spaced between two extremal points. Beume et al.~\cite[Lemma~1]{BeumeFLPV09} proved this result for the line $\{(x,\beta-x) \mid x\in[x_{\min},x_{\max}]\}$ for any $\beta > 0$ and $x_{\min},x_{\max} \in \R$ with $x_{\min} \le x_{\max}$. Auger et al.~\cite{AugerBBZ12} further extended it to the line $\{(x,\beta-\alpha x) \mid x\in[x_{\min},x_{\max}]\}$ with $0\le x_{\min} < x_{\max} \le \beta/\alpha$ for any $\alpha,\beta > 0$. They also considered the case that the reference point is dominated by at least one Pareto optimal point, but not necessarily all. 
Nguyen, Sutton, and Neumann~\cite{NguyenSN15} discussed two special cases, namely $\mu=\sqrt n+1$ for \omm and $\mu \in (1,n/3)$ for \lotz. Both implicitly assume $n/(\mu-1)$ to be an integer. 
Doerr, Gao, and Neumann~\cite{DoerrGN16} compared the HV of the computed approximations with the HV of the full Pareto front and thus did not require an analysis of the best HV achievable with fixed small population sizes. 
In summary, for the discrete \omm problem, the optimal HV for $\mu$ points is not well analyzed. Hence, we will also bound the optimal HV value when we discuss the relation of MEI and HV in the following Section~\ref{sssec:meihv}.

\subsubsection{Relation of Maximal Empty Intervals and Hypervolumes}\label{sssec:meihv}

We now analyze the relation between the MEI and the $\hv$. As we shall see, the fact that the hypervolume depends on the sum of the squares of the lengths of the empty intervals disallows a simple relation between the two measures. However, we shall also see that the optimal approximations for the two measures are very similar, in particular, an optimal approximation with respect to the $\hv$ has the minimum possible empty interval size.

We first give an exact formula for the \emph{hypervolume}. Let $r=(r_1,r_2)$ with $r_1,r_2\le0$ be the reference point so that it is dominated by any $x\in\{0,1\}^n$ w.r.t. \omm. 
Let $S$ be a subset of the Pareto front $M$ of \omm with $(0,n),(n,0)\in S$, and $J$ be the set of the first objective values of $S$. More specifically, let $J=\{j_1,\dots,j_{|J|}\}$ with $j_a<j_b$ for $a < b$. Then we know $j_1=0$ and $j_{|J|}=n$. Since each Pareto front point $(a,b)$ satisfies $a+b=n$, we have
\begin{equation}
\begin{split}
\hv(f(S),r)&={}\left(-r_1n-r_2n+r_1r_2+\tfrac12 n^2\right)-\sum_{a=2}^{|J|} \tfrac12 (j_a-j_{a-1})^2\\
&=:{}A-\sum_{a=2}^{|J|} \tfrac12 (j_a-j_{a-1})^2,
\label{eq:hvsr}
\end{split}
\end{equation}
where we note that $A=-r_1n-r_2n+r_1r_2+\tfrac12 n^2$ is the HV of the continuous front $\{(x,n-x)\mid x\in[0,n]\}$. We point out that the $A$-part was wrongly computed as $A=-r_1n-r_2n+2r_1r_2+\tfrac12 n^2$ in~\cite[Lemma~2]{DoerrGN16}. This renders the calculations of the expressions $I_H(F)$ and $I_H(P)$ there incorrect, but this does not influence the correctness of Lemma~2 in \cite{DoerrGN16} as the extra $r_1r_2$ term cancels out in $I_H(F)-I_H(P)$.

Now we estimate $\hv(f(S),r)$.
\begin{lemma}\label{lem:mei2hv}
Let $S$ be a subset of the Pareto front $M$ of \omm with $(0,n),(n,0)\in S$, and let $r=(r_1,r_2)$ with $r_1,r_2\le 0$ be the reference point. Let $A=-r_1n-r_2n+r_1r_2+\frac12 n^2$. Then 
\begin{align*}
\hv(f(S),r)\in \left[A- \frac{(|S|-1)(\mei(S))^2}{2}, A-\frac{n}{2(|S|-1)}\right].
\end{align*} 
\end{lemma}
\begin{proof}
We use the same $J$ as defined above. Then we have $|J|\le |S|$. By the definition of $\mei(S)$, we have
\begin{equation}
\label{eq:hvl}
\begin{split}
\sum_{a=2}^{|J|} &{}\tfrac12 (j_a-j_{a-1})^2\le \sum_{a=2}^{|J|} \tfrac12 (\mei(S))^2 \\
&\le{} \sum_{a=2}^{|S|} \tfrac12 (\mei(S))^2=\tfrac12 (|S|-1)(\mei(S))^2.
\end{split}
\end{equation}
From the Cauchy–Bunyakovsky–Schwarz inequality, we have
\begin{equation}
\label{eq:hvu}
\begin{split}
\sum_{a=2}^{|J|} &{}\tfrac12 (j_a-j_{a-1})^2=\frac1{2(|J|-1)} \left(\sum_{a=2}^{|J|}  (j_a-j_{a-1})^2\right)\left(\sum_{a=2}^{|J|} 1^2\right) \\
&\ge{} \frac1{2(|J|-1)} \left(\sum_{a=2}^{|J|} j_a-j_{a-1} \right)^2 =\frac1{2(|J|-1)}(j_{|J|}-j_1)^2\\
&={} \frac{n}{2(|J|-1)} \ge  \frac{n}{2(|S|-1)},
\end{split}
\end{equation}
where the last equality uses $j_{|J|}=n$ and $j_1=0$. 

From (\ref{eq:hvsr}) to (\ref{eq:hvu}), this lemma is proved.
\end{proof}

Similar to $\mei_{\opt}$, we define an optimal value for the hypervolume. From Definition~\ref{def:hv}, it is not difficult to see that the larger $\hv$, the better $S$ approximates the Pareto front. 
For $N\in\N_{\ge 2}$, we further define
\begin{align*}
\hv_{\opt}(N,r):=\sup\{\hv(f(S),r) \mid S\subseteq M, |S|\le N, (0,n) \in S, (n,0) \in S\}.
\end{align*}
Obviously, this is the largest $\hv$ that a population of size $N$ containing the two extremal points can obtain. Now we estimate $\hv_{\opt}(N,r)$.
\begin{lemma}
Let $r=(r_1,r_2)$ with $r_1,r_2\le 0$ be the reference point, and let $A=-r_1n-r_2n+r_1r_2+\frac12 n^2$. For all $N\in \N_{\ge 2}$, we have 
\begin{align*}
\hv_{\opt}(N,r)\in \left[A-\frac{(N-1)(\mei_{\opt}(N))^2}{2}, A-\frac{n}{2(N-1)}\right].
\end{align*}
\label{lem:optimalhv}
\end{lemma}
\begin{proof}
For any $S\subseteq M$ with $|S|\le N, (0,n) \in S$, and $(n,0) \in S$, Lemma~\ref{lem:mei2hv} shows that 
\begin{align*}
\hv(S,r)\le A-\frac{n}{2(|S|-1)} \le A-\frac{n}{2(N-1)},
\end{align*}
where the second inequality uses $|S|\le N$. Thus $\hv_{\opt}(N,r)\le A-n/(2(N-1))$.

Now we show the lower bound of $\hv_{\opt}(N,r)$. Let $j_i=\min\{(i-1)\left\lceil n/(N-1)\right\rceil,n\}, i=1,\dots, N$, and we know that $j_1=1$ and $j_N=n$. Consider the set $S' = \{(j_i,n-j_j) \mid i \in [1..N]\}$. 
It is not difficult to see that
\begin{align*}
\min\left\{i\left\lceil\tfrac{n}{N-1}\right\rceil,n\right\}-\min\left\{(i-1)\left\lceil\tfrac{n}{N-1}\right\rceil,n\right\}\le \left\lceil\tfrac{n}{N-1}\right\rceil.
\end{align*}
Hence, from (\ref{eq:hvsr}) we have
\begin{align*}
\hv(S',r)&={}A-\sum_{a=2}^N \tfrac12 (j_a-j_{a-1})^2 \\
&\ge{} A-\sum_{a=2}^N \tfrac12\left(\left\lceil\tfrac{n}{N-1}\right\rceil\right)^2 = A-\tfrac12(N-1)\left(\left\lceil\tfrac{n}{N-1}\right\rceil\right)^2.
\end{align*}
By the definition of $\hv_{\opt}(N,r)$ and noting $\mei_{\opt}(N)=\left\lceil n/(N-1)\right\rceil$ from Lemma~\ref{lem:optimalmei}, the lower bound is proven.
\end{proof}
Different from Corollary~\ref{cor:epsupper} and~\ref{cor:epslow}, we do not see a simple relation between $\mei_{\opt}(N)$ and $\hv_{\opt}(N,r)$ here. However, we note that the construction of $S'$ here is the same as the one in the proof of Lemma~\ref{lem:optimalmei}. Hence, the lower bound of the $\hv_{\opt}(N,r)$ is reached when $\mei_{\opt}(N)$ is reached. 

Since our MEI measure is more intuitive and crucial in our proofs, and it can be easily transferred to the $\eps$-approximation and HV measure that are utilized in the community, we will use MEI as the approximation measure to see how well an MOEA approximates the Pareto front in the remaining of this work.

\section{Difficulties for the \NSGA to Approximate the Pareto Front}
\label{sec:nsgaii}

In this section, we show that the traditional way how the \NSGA selects the next population, namely by relying on the initial crowding distance, can lead to suboptimal approximations of the Pareto front. To this aim, we analyze by mathematical means the results of the selection from two different combined parent and offspring populations. These examples demonstrate quite clearly that the traditional selection can lead to unwanted results. We note that these results do not prove completely that the \NSGA has difficulties to find good approximations since we do not know how often the \NSGA enters exactly these situations. Unfortunately the population dynamics of the \NSGA are too complicated for a full proof. Our experimental results in Section~\ref{sec:exp}, however, suggest that the phenomena we observe in these synthetic situations (in particular, the first one) do show up. 

We start by regarding the optimal-looking situation that the combined parent and offspring population for each point on the Pareto front contains exactly one individual. Hence by removing the individual corresponding to (essentially) every second point on the Pareto front, one could obtain a very good approximation of the front. Surprisingly, the \NSGA does much worse. Since all solutions apart from the two extremal ones have the same crowding distance, the \NSGA removes a random set of $N$ out of these $2N-2$ inner solutions. As we show now, with high probability this creates an empty interval on the Pareto front of length $\Theta(\log n)$. 

\begin{lemma}
Let $n\ge 7$ be odd. Consider using the \NSGA to optimize the \omm function with problem size $n$. Let the population size be $N=(n+1)/2$. Suppose that in a certain generation $t\ge 0$, the combined parent and offspring population $R_t$ fully covers the Pareto front, that is, $f(R_t)=\{(0,n),(1,n-1),\dots,(n,0)\}$. Then with probability at least $1 - \exp(-\Omega(n^{1/2}))$, we have $\mei(P_{t+1},f_1) \ge \lfloor \frac 13 \ln n \rfloor$. On the positive side, $E[\mei(P_{t+1},f_1)]\le \log_{3/2} n+O(1)$ and $\Pr[\mei(P_{t+1},f_1) \ge c\log_{3/2} n] \le n^{1-c}$ for any constant $c > 1$.
\label{lem:bigMEI}
\end{lemma}

We note that in the above result, we did not try to optimize the implicit constants. The proof of the upper bound on the length of the longest empty interval is simple union bound argument. The proof of the lower bound needs some care to deal with the stochastic dependencies among the intervals. We overcome these difficulties by only regarding a set of disjoint intervals and with a two-stage sampling process such that the first stage treats the regarded intervals in an independent manner. 

\begin{proof}
We note that $R_t$ has size at least $n+1$, since it covers the Pareto front, which has a size of $n+1$. From our assumption $N = (n+1)/2$, we thus conclude that every point in the Pareto front has exactly one corresponding individual in $R_t$. Hence, except for the individuals $0^n$ and $1^n$ that have infinite crowding distance, all other individuals have the equal crowding distance of $4/n$. Consequently, the original \NSGA survival selection will randomly select $N$ individuals from the $2N-2 = n-1$ inner individuals to be removed. 

To prove the logarithmic lower bound on the expectation, we argue as follows. Let $k = \lfloor \frac 13 \ln n \rfloor$. Let $M \subseteq [1..n-k]$ such that $|M| = \lceil n / \ln n \rceil$ and for any two $m_1, m_2 \in M$, we have $|m_1 - m_2| \ge k$ or $m_1 = m_2$ (note that such a set exists since $k|M| \le n-1$ by definition of $k$ and $M$). Consequently, for different $m \in M$, the intervals $I_m = [m .. m+k-1]$ are disjoint subsets of $[1..n-1]$. 

To cope with the stochastic dependencies, we regard a particular way to sample the $N$ individuals to be removed. In a first phase, we select each individual $x$ with $f_1(x) \in [1..n-1]$ with probability $N / (2(n-1))$. This defines a random set $X$ of individuals with expected cardinality $E[|X|] = N/2$. We remove these individuals from the combined parent and offspring population. In a second phase, if $|X| < N$, which is very likely, then we repeat removing random individuals $x$ with $f_1(x) \in [1..n-1]$ until we have removed a total of $N$ individuals. If $|X| > N$, then we repeat adding random previously removed individuals until we have removed only $N$ individuals.

We now prove that with high probability, the first phase ends with (i)~an interval~$I$ of length $k$ on the Pareto front which is not covered by the population, and with (ii)~$|X| \le N$. Apparently, then $I$ also in the final population is not covered. Denote by $\tilde A_{m,k}$ the probability that all individuals $x$ with $f_1(x) \in I_m$ are removed in the first phase. We have
\begin{align*}
\Pr[\tilde A_{m,k}] &= \left(\tfrac{N}{2(n-1)}\right)^k \ge 4^{-k} \ge n^{-\ln(4)/3}.
\end{align*}
By construction, the events $\tilde{A}_{m,k}$, $m \in M$, are independent. Hence, using the estimate $1+r \le \exp(r)$ valid for all $r \in \R$, we estimate
\begin{align*}
\Pr[\forall m \in M : \neg \tilde A_{m,k}] &\le (1 - n^{-\ln(4)/3})^{|M|} \le \exp\left(- n^{-\ln(4)/3}\frac{n}{\ln n}\right)\\
&=\exp\left(- \frac{n^{1-\ln(4)/3}}{\ln n}\right) = \exp(-\Omega(n^{1/2})). 
\end{align*}
Finally, we estimate the probability that $|X| > N$. We note that $|X|$ can be written as a sum of $n-1$ independent binary random variables. Hence by the multiplicative Chernoff bound (Theorem~1.10.1 in~\cite{Doerr20bookchapter}), we have
\[
\Pr[|X| > N] \le \Pr[|X| \ge 2 E[X]] \le \exp(-E[X] / 3) \le \exp(-n/6).
\] 
This proves that with probability at least $1 - \exp(-n^{1-\ln(4)/3} / \ln n) - \exp(-n/6) = 1 - \exp(-\Omega(n^{1/2}))$, after the  selection phase there is an interval of length $k$ of the Pareto front such that none of its points is covered by the new population $P_{t+1}$.  This also implies $E[\mei(P_{t+1},f_1)] \ge (1 - \exp(-\Omega(n^{-1/2}))) k = \Omega(\log n)$.

We now turn to the upper bounds. Let $k\in[1..N]$ and $m\in[1..n-k]$, and let $A_{k,m}$ be the event that all individuals with $f_1$ value in $I_m = [m..m+k-1]$ are selected for removal. Then 
\begin{align*}
\Pr[A_{m,k}]=\frac{\binom{2N-2-k}{N-k}}{\binom{2N-2}{N}}=\frac{\frac{(2N-2-k)!}{(N-k)!(N-2)!}}{\frac{(2N-2)!}{N!(N-2)!}} = \frac{N(N-1)\cdots(N-k+1)}{(2N-2)(2N-3)\cdots(2N-1-k)}.
\end{align*}
It is not difficult to see that if $A_{m,k}$ happens for some $m$, then $\mei(P_t,f_1)\ge k$. Hence, by a union bound over all possible $m$, we obtain
\begin{align*}
\Pr[\mei(P_{t+1},f_1) \ge k] &\le{} (n-k)\Pr[A_{m,k}] = \frac{(n-k)N(N-1)\cdots(N-k+1)}{(2N-2)(2N-3)\cdots(2N-1-k)}\\
&\le{} (n-k) \left(\frac{N}{2N-2}\right)^k \le n\left(\frac{1}{2-2/N}\right)^k \le n\left(\frac23\right)^k,
\end{align*}
where the last inequality uses $n\ge 7$ and thus $N=(n+1)/2\ge 4$. Hence for any constant $c>1$, we know that 
\begin{align*}
\Pr[\mei(P_{t+1},f_1) \ge c\log_{3/2}n]\le n\left(\tfrac23\right)^{c\log_{3/2}n}=\tfrac1{n^{c-1}}.
\end{align*}
For the expectation, we compute
\begin{align*}
E&{}[\mei(P_{t+1,f_1})] =\sum_{k=1}^{+\infty} \Pr[\mei(P_{t+1},f_1) \ge k] \\
&\le{}\sum_{k=1}^{\lfloor\log_{3/2}n\rfloor} \Pr[\mei(P_{t+1},f_1) \ge k] + \sum_{k=\lfloor\log_{3/2}n\rfloor+1}^{+\infty} n\left(\tfrac23\right)^k\\
&\le{} \lfloor\log_{3/2}n\rfloor + n\left(\tfrac23\right)^{\lfloor\log_{3/2}n\rfloor+1} \sum_{k=0}^{+\infty} \left(\tfrac23\right)^k \le \log_{3/2}n+O(1).
\qedhere
\end{align*}
\end{proof}

The example above shows that even in a perfectly symmetric situation, the \NSGA with high probability selects a new parent population with high irregularities and relatively large areas on the Pareto front that are not covered by the population. 

We now show that even more extreme examples can be constructed. This example builds on the same idea as the $11$-point example in Figure~2 in~\cite{KukkonenD06}, but is much more general (working for arbitrary large population sizes) and gives a more extreme result. We do not expect these to come up often in a regular run of the \NSGA, but they underline that the drawbacks of working with the initial crowding distance can be tremendous. 

\begin{lemma}
For all $n \in 3\N$, there is a combined parent and offspring population $R$ with $0^n, 1^n \in R$ and the following two properties.
\begin{itemize}
\item The population $P'$ selected by the traditional \NSGA satisfies $\mei(P',f_1) = \frac 13 n+2$.
\item There is a population $P'' \subseteq R$ with $|P''|=N$ and $0^n, 1^n \in P''$ such that $\mei(P'',f_1) \le 4$.
\end{itemize}
\label{lem:bigMEI2}
\end{lemma}

\begin{proof}
Let $R$ be such that $f_1(R) = [0..\frac 13n+1] \cup \{\frac 13 n + 2i \mid i \in [1..\frac 13 n]\}$ and $|R| = 2 (\frac 13 n + 1)$. Then for any two different $x,x'\in R$, we have $f_1(x) \neq f_1(x')$. By construction, $0^n, 1^n \in R$. We note that exactly those $x \in R$ with $f_1(x) \in [1..\frac 13n+1]$ have the smallest occurring crowding distance of $\frac 4 {n}$. Since these are $\frac 13n+1$ elements of $R$ and since $N = \frac 12 |R| = \frac 13n+1$ elements that have to be removed, they will all be removed in the selection step, leaving all points $(i,n-i)$ with $i \in [1..\frac 13n+1]$ uncovered by the selected population.  

Now let $P''$ be obtained from $R$ by keeping $0^n$, removing the individuals with $f_1$ values $1$ and $2$, keeping the individual with $f_1$ value $3$, and then alternatingly in order of increasing $f_1$ value deleting and keeping the individuals. This removes exactly half the individuals from $R$, but leaves $0^n$ and $1^n$ in $P''$. Apart from the interval $[0..3]$, the intervals in $f_1(P'')$ are all the union of two intervals in $f_1(R)$. Since $\mei(R,f_1) = 2$, this shows $\mei(P'',f_1) \le 4$. 
\end{proof}

\section{Working With the Current Crowding Distance}\label{sec:onthefly}

As already noted in~\cite{KukkonenD06} and then further quantified in the preceding section, the traditional way the \NSGA selects the next parent population can lead to unevenly distributed populations. The natural reason, discussed also in~\cite{KukkonenD06} and made again very visible in the proofs of Lemmas~\ref{lem:bigMEI} and~\ref{lem:bigMEI2}, is that the crowding distance is computed only once and then used for all removals of individuals. Consequently, it may happen that individuals are removed which, at the time of their removal, have a much larger crowding distance than at the start of the selection phase.

This phenomenon is heavily exploited in the construction of the very negative example in the proof of Lemma~\ref{lem:bigMEI2}. In this example, the individuals $x$ with $f_1(x) \in [1..\frac 13n]$ all have the smallest crowding distance of $4/n$, and thus are all removed in some arbitrary order. When the last individual is removed, its neighbors on the front have objective values $(0,n)$ and $(\frac 13 n+1,n - (\frac 13n + 1))$. Consequently, this individual at the moment of its removal has a crowding distance of 
$2/3+2/n$,
 which is (for large $n$) much larger than its initial value of $4/n$, but also much larger than the crowding distance of $8/n$, which most of the remaining individuals still have. This example shows very clearly the downside of working with the initial crowding distance and at the same time suggests to work with the current crowding distance instead. 

This is the road we will follow in this section. We first argue that there is an efficient implementation of the \NSGA that repeatedly selects an individual with smallest crowding distance. We then show that this selection mechanism leads to much more balanced selections for the \oneminmax benchmark. We prove that the modified \NSGA achieves an MEI value of at most $\frac{2n}{N-3}$, which is only by roughly a factor of $2$ larger than the optimal MEI of $(1 + o(1)) \frac n N$. Reaching such a balanced distribution is very efficient -- once the two extremal points are found (in time $O(Nn\log (n))$), it only takes additional time $O(Nn)$ to find a population satisfying the MEI guarantee above. From this point on, the MEI never increases above $\frac{2n}{N-3}$. 

\subsection{Implementation of the \NSGA Using the Current Crowding Distance}

It is immediately obvious that the computation of the crowding distance in the original \NSGA can be implemented in a straightforward manner that takes $O(m N \log N)$ time, where $m$ is the number of objectives and $N$ is the number of individuals for which the crowding distance is used as tie-breaker. 

For the \NSGA using the current crowding distance, some more thought is necessary. A na\"ive computation of the crowding distance for each removal could lead to a total effort of $\Theta(m N^2 \log N)$. Since the removal of one individual can change the crowding distance of the at most $2m$ individuals which are its neighbors in one of the sortings, one would hope that more efficient implementations exist, and this is indeed true. 

In fact, already in \cite{KukkonenD06} this problem was discussed and a solution via a priority queue and an additional data structure storing the neighbor relation was presented. Since the presentation in \cite{KukkonenD06} is very compact (and did not discuss how to implement a random tie-breaking among individuals with the same crowding distance), we now describe this approach in more detail.

%

For the following presentation, let us assume that we have only $m=2$ objectives. Let us assume that we have a set $R$ of individuals which pairwise are not comparable (none dominates the other) or have identical objective values. When optimizing $\oneminmax$, this set $R$ will be the combined parent and offspring population; in the general case it will be the front $F_{i^*}$. Let us call $r = |R|$ the size of this set. Let us assume that we want to remove some number $s$ of individuals from $R$, sequentially by repeatedly removing an individual with the smallest current crowding distance, breaking ties randomly.

Besides keeping the crowding distance updated (which can be done in constant time per removal, as we just saw), we also need to be able to efficiently find and remove an individual with the smallest (current) crowding distance, and moreover, a random one in case of ties. The detection and removal of an element with the smallest key calls for a \emph{priority queue}. Let us ignore for the moment the random tie-breaking and only discuss how to use a priority queue for the detection and removal of an individual with the smallest crowding distance. We recall that a priority queue is a data structure that stores \emph{items} together with a \emph{key}, a numerical value assigned to each item that describes the priority of the item. Standard priority queues support at least the following three operations: Adding new items to the queue, removing from the queue an item with the smallest key, and decreasing the key of an item in the queue. They do so with a time complexity that is only logarithmic in the current length of the queue. 

For our problem, at the start of the selection phase, we add all individuals with their crowding distance as key to the priority queue. We repeatedly remove individuals according to the following scheme: (i)~We find and remove from the priority queue an individual $x$ with smallest key. We also remove $x$ from $R$. (ii)~We find the up to four neighbors of $x$ in the two sortings of $R$ according to the two objectives, compute their crowding distance, and update their keys in the priority queue accordingly. This is not a decrease-key operation (but an increase of the key), but such an increase-key can be simulated by first decreasing the key to an artificially small value (smaller than all real values that can occur, here for example $-1$), then removing the item with smallest key (which is just this item), and then adding it with the new key to the queue. 

These two steps can be implemented in logarithmic time, since all operations of the priority queue take at most logarithmic time, except that we still need to provide an efficient way to find the neighbors of an element in the sortings. This is necessary to determine the up to four neighbors of $x$, but also to compute their crowding distance (which needs knowing their neighbors). To enable efficient computations of such neighbors, we use an additional data structure, namely for each objective a doubly-linked list that stores the current set $R$ sorted according to this objective. This list data structure must enable finding predecessors and successors (provided they exist) as well as the deletion of elements. Standard doubly-linked lists support these operations in constant time. We use this list in step~(ii) above to find the desired neighbors. We also need to delete the removed individual $x$ in step~(i) from this list. 

To allow finding individuals in the priority queue or the doubly-linked list, we need a helper data structure with pointers to the individuals in these data structures. This can be a simple array indexed by the initial set $R$. When $R$ is not suitable as index, we can initially build an array $a[1..|R|]$ storing $R$ and then use the indices of the elements of $R$ in $a$ as identifiers. With this approach, regardless of the structure of $R$, we can access individuals in any of our data structures in constant time. 

So far we have described how to repeatedly remove an element with the currently smallest crowding distance each in logarithmic time. To add the random tie-breaking, it suffices to give each individual $x$ a random second-priority key, e.g., a random number $r_x \in [0,1]$. The key of an individual $x$ used in the priority queue now is composed of the current crowding distance and this number $r_x$, where a key is smaller than another when its crowding distance is smaller or when the crowding distances are equal and the $r_x$ number is smaller (lexicographic order of the two parts of the key). With these extended keys, the individuals with the currently smallest crowding distance have the highest priority, and in the case of several such individuals, the number $r_x$ serves as a random tie-breaker. 

With this tie-breaking mechanism, we have now implemented a way to repeatedly remove an individual with the smallest current crowding distance, breaking ties randomly, in time logarithmic in $r$, the size of the set $R$. Overall, our selection using the current crowding distance takes the same time of $O(r \log r)$ as the selection based on the initial crowding distance. Note that both complexities are small compared to the non-dominated sorting step, which takes time quadratic in~$N$. 

Without going into details, we note that when the possible crowding distance values are known and they are not too numerous, say there is an upper bound of $S$ for their number, then using a bucket queue instead of a standard priority queue would give a complexity of order $O(S + r)$. This is slightly superior to the above runtime, e.g., for \oneminmax when $N = \Omega(n/\log n)$. 


\begin{algorithm}[!ht]
    \caption{\NSGA with the survival selection using the current crowding distance}
    \begin{algorithmic}[1]
    \STATE {Uniformly at random generate the initial population $P_0=\{x_1,x_2,\dots,x_N\}$ with $x_i\in\{0,1\}^n,i=1,2,\dots,N.$}
    \FOR{$t = 0, 1, 2, \dots$}
    \STATE {Generate the offspring population $Q_t$ with size $N$}
    \STATE {Use fast-non-dominated-sort() in~\cite{DebPAM02} 
    to divide $R_t$ into fronts $F_1,F_2,\dots$}
    \STATE {Find $i^* \ge 1$ such that $|\bigcup_{i=1}^{i^*-1}F_i| < N$ and $|\bigcup_{i=1}^{i^*}F_i| \ge N$}
    \STATE {Use Algorithm~\ref{alg:cDis} to separately calculate the crowding distance of each individual in $F_{1},\dots,F_{i^*}$ }
    \STATEx {\textsl{~~\%\%~Survival selection using the current crowding distance}}
    \WHILE{$|\bigcup_{i=1}^{i^*}F_{i}|\neq N$}\label{ste:flybegin}
    \STATE {Let $x$ be the individual with the smallest crowding distance in $F_{i^*}$, chosen at random in case of a tie}
    \STATE {Find four neighbors of $x$, two in the sorted list with respect to $f_1$ and two for $f_2$. Update the crowding distance of these four neighbors}
    \STATE {$F_{i^*} = F_{i^*} \setminus \{x\}$}
    \ENDWHILE\label{ste:flyend}
    \STATE {$P_{t+1}=\left(\bigcup_{i=1}^{i^*}F_i\right)$}
    \ENDFOR 
    \end{algorithmic}
    \label{alg:onthefly}
\end{algorithm}

\subsection{Runtime Analysis and Approximation Quality of the \NSGA with Current Crowding Distance}\label{subsec:approx}

We now conduct a mathematical analysis of the \NSGA with survival selection based on the current crowding distance. The following lemma shows that invididuals with at least a certain crowding distance will certainly enter the next generation. The key argument in this proof is an averaging argument based on the observation that the sum of the crowding distances of all individuals other than the ones with infinite crowding distance is at most $4$.

\begin{lemma}\label{lem:selection}
Let $N\ge 4$.
Consider using the \NSGA with survival selection based on the current crowding distance to optimize the \omm function with problem size $n$. Let $t_0$ be the first generation such that the two extreme points $0^n$ and $1^n$ are  in $P_{t_0}$. Let $t\ge t_0$. Consider the selection of the next population $P_{t+1}$ from $R_t$, which consists of $N$ times removing an individual with smallest current crowding distance, breaking ties in an arbitrary manner. Assume that at some stage of this removal process the individual $x$ has a crowding distance of $\frac{4}{N-3}$ or higher. Then $x \in P_{t+1}$. Also, $P_{t+1}$ surely contains the two extreme points.
\end{lemma}

\begin{proof}
We consider first a single removal of an individual at some moment of the selection phase. Let $R$ be the set of individuals from the combined parent and offspring population that have not yet been removed. Note that $r := |R| > N$. Note also that, by definition of the crowding distance, there can be at most $4$ individuals with infinite crowding distance. Since $N \ge 4$, such individuals are never removed, and consequently, $R$ surely contains $0^n$ and $1^n$. Let $y_1, \dots, y_r$ be the sorting of $R$ by increasing $f_1$ value and $z_1, \dots, z_r$ be the sorting of $R$ by increasing $f_2$ value that are used for the computation of the crowding distance. 
For $x \in R$, let $i, j \in [1..r]$ such that $x = y_i = z_j$ (we assume here that individuals have unique identifiers, so they are distinguishable also when having the same genotype; consequently, these $i$ and $j$ are uniquely defined). From the definition of $\cDis(x)$, we know the following. If $\{i,j\} \cap \{1,r\} \neq \emptyset$, then $\cDis(x) = \infty$. Otherwise, $\cDis(x) = (f_1(y_{i+1}) - f_1(y_{i-1}) + f_2(z_{j+1}) - f_2(z_{j-1}))/n$. 

We compute
\begin{align*}
\sum_{i=2}^{r-1} f_1(y_{i+1}) - f_1(y_{i-1}) 
&= \sum_{i=2}^{r-1} f_1(y_{i+1}) - f_1(y_i) + f_1(y_i) - f_1(y_{i-1}) \\
&\le 2 \sum_{i=1}^{r-1} f_1(y_{i+1}) - f_1(y_i) \\
&= 2 (f_1(y_r) - f_1(y_1)) = 2 (n- 0) = 2 n,
\end{align*}
the latter by our insight that $R$ contains the two extremal individuals. An analogous estimate holds for $f_2$ and the $z_j$. This allows the estimate
\begin{align*}
\sum_{x \in R^*} n\cDis(x)
\le \sum_{i=1}^{r-1} f_1(y_{i+1}) - f_1(y_{i-1}) + \sum_{j=1}^{r-1} f_2(z_{j+1}) - f_2(z_{j-1})  \le 4n,
\end{align*} 
where we write $R^*$ for the individuals in $R$ having a finite crowding distance. Since $|R^*| \ge r - 4$, a simple averaging argument shows that there is an $x \in R^*$ with $\cDis(x) \le \frac{4}{r-4}$. Hence also the individual that is removed has a crowding distance of at most $\frac{4}{r-4} \le \frac{4}{N-3}$. 

Now looking at all removals in this selection step, we see that never an individual with crowding distance above $\frac{4}{N-3}$ is removed. 
\end{proof}

The result just shown easily implies that no large empty interval is created by the removal of a solution from $R_t$ in the selection phase. Slightly more involved arguments are necessary to argue that the MEI-value decreases relatively fast. It is not too difficult to see that if the set of $f_1$ values in $P_t$ contains a large empty interval (the same will then be true for $f_2$), then this interval can be reduced in length by at least one via the event that one of the individuals corresponding to the boundaries of the empty interval create an offspring ``inside the interval''. What needs more care is that we require such arguments for all large empty intervals in parallel. For this, we regard how an empty interval shrinks over a longer time. Since this shrinking is composed of independent small steps, we can use strong concentration arguments to show that an interval shrinks to the desired value in the desired time with very high probability. This admits a union bound argument to extend this result to all empty intervals. 

A slight technical challenge is to make precise what ``one interval'' means over the course of time (recall that in one iteration, we generate $N$ offspring and then remove $N$ individuals from $R_t$). We overcome this by regarding a half-integral point $i+0.5$ and the corresponding interval 
\begin{align*}
I_i = [\max\{f_1(x) \mid x \in P_t, f_1(x) \le i+0.5\}..\min\{f_1(x) \mid x\in P_t, f_1(x) \ge i+0.5\}].
\end{align*}
This definition is unambiguous. It gives room to some strange effects, which seem hard to avoid, nevertheless. Assume that $I_i = [i..b]$ for some $b$ sufficiently larger than $i$. Assume that there is a single individual $x$ with $f_1(x) = i$ in $P_t$. Assume that $x$ has an offspring $y$ with $f_1(y) = i+1$. If both $x$ and $y$ survive into the next generation, then $I_i$ has suddenly shrunk a lot to $[i..i+1]$. If $y$ survives but $x$ does not, then 
$I_i$ has changed considerably to $[a..i+1]$, where $a = \max\{f_1(x) \mid x \in P_t, f_1(x) < i\}$. Since in each iteration $N$ individuals are newly generated and then $N$ are removed, we do not see a way to define the empty intervals in a more stable manner. Nevertheless, our proof below will be such that it covers all these cases in a relatively uniform manner. 

We recall from Section~\ref{sec:pre} that \emph{fair selection} means that each individual of the parent population generates exactly one offspring and \emph{random selection} means that $N$ times an individual is chosen uniformly at random from the parent population to generate one offspring. 

\begin{lemma}
Let $N\ge 4$.
Consider using the \NSGA with fair or random parent selection, with one-bit or bit-wise mutation, and with survival selection using the current crowding distance, to optimize the \omm function with problem size $n$. Let $t_0$ be the first generation such that the two extreme points $0^n$ and $1^n$ are contained in $P_{t_0}$. Let $t_1 \ge t_0$ be the first generation such that $\mei(P_{t_1},f_1)\le \max\{\frac{2n}{N-3},1\} =: L$. Then $t_1 - t_0 = O(n)$, both in expectation and with probability $1-o(1)$. Also, for all $t > t_1$, we have $\mei(P_t,f_1) \le L$ with probability one.
\label{lem:approx}
\end{lemma}

\begin{proof}
Let $i\in[0..n-1]$. For all $t\ge t_0$, let $X_t$ be the length of the empty interval in $f_1(P_t)$ containing $i+0.5$, that is, 
\begin{align*}
X_t=\min\{f_1(x) \mid x\in P_t, f_1(x) \ge i+0.5\}-\max\{f_1(x) \mid x \in P_t, f_1(x) \le i+0.5\},
\end{align*}
and let analogously $Y_t$ be the length of the empty interval in $f_1(R_t)$ that contains $i+0.5$. 

We first prove if for some $t \ge t_0$ and $M \ge L$ we have $Y_t \le M$, then we have $X_{t'} \le M$ for all $t' > t$ with probability one. It apparently suffices to regard one iteration, so assume by way of contradiction that $X_{t+1} > M$. Note that $P_{t+1}$ is obtained from $R_t$ by $N$ times removing an individual with smallest current crowding distance. Consider the removal step which extends the empty interval containing $i+0.5$ to a length larger than $M$. Let $[a..b]$ be the empty interval containing $i+0.5$ before this removal. To increase the size of this empty interval, the individual $x$ removed in this step must have $f_1(x) \in \{a, b\}$ and it must be the only individual with this $f_1$ value. Assume, without loss of generality, that $f_1(x) = a$. Removing $x$, by definition of the crowding distance, creates an empty interval of length exactly $n/2$ times the current crowding distance of $a$ (here the factor of $1/2$ stems from the fact that for an individual $x$ with unique $f_1$ and $f_2$ value, the contributions of the two objectives to the crowding distance are equal). By Lemma~\ref{lem:selection}, this current crowding distance of $x$ is at most $\frac{4}{N-3}$, contradicting our assumption that the removal of $a$ creates an empty interval of length larger than $L$. This proves our first claim. 

Since $X_t \ge Y_t$ for all $t \ge t_0$, this claim implies that $X_t\ge X_{t+1}$ for $X_{t}\ge L$, and therefore that once $X_t \le L$, we have $X_{t'} \le L$ for all $t' \ge t$. Consequently, for all $t > t_1$, we have $\mei(P_t,f_1) \le L$ with probability one.

It remains to show that $X_t$ is actually decreasing to at most $L$ over time. To this aim, we now consider the situation that $X_t>L$ for some $t \ge t_0$. Let $a$ and $b$ with $a\le b$ be the two ends of the interval in $f_1(P_t)$ that contains $i+0.5$. Let $x_a,x_b\in P_t$ respectively be individuals with $f_1$ values $a$ and $b$. The probability that mutation applied to $x_a$ creates an offspring $y$ with $f_1(y) = a+1$ is $\frac{n-a}{n}$ for one-bit mutation and $(1-\frac 1n)^{n-1} \frac {n-a}{n} \ge \frac {n-a}{en}$ for bit-wise mutation. Likewise, the probability that $x_b$ mutates into a $y$ with $f_1(y) = b-1$ is $\frac bn$ for one-bit mutation and at least $\frac {b}{en}$ for bit-wise mutation. Since $a \le b$, we have $(n-a)+b \ge n$, that is, for at least one of $x_a$ and $x_b$ from which a $y$ with $f_1(y)\in \{a+1,b-1\}$ is mutated, such probability is at least $\frac {1}{2e}$. The probability that this individual is chosen as parent at least once in this iteration is one for fair parent selection and $(1 - (1 - \frac 1N)^N) \ge 1 - \frac 1e$ for random parent selection. Consequently, the probability that at least one individual with $f_1$ value in $\{a+1,b-1\}$ is created in this iteration is at least $\frac{1}{2e} (1 - \frac 1e) =: p$.
%
%

Let us assume that this positive event has happened. Then the interval in $f_1(R_t)$ containing $i+0.5$ has length $Y_t \le X_t - 1$ (note that, depending on the value of $i$ and the outcome of the offspring generation, this interval can be any of $[a..a+1]$, $[a..b-1]$, $[a+1..b-1]$, $[a+1..b]$, and $[b-1..b]$, but all arguments below hold for any of these cases). By our first claim, we now have $X_{t+1} \le X_t - 1$ with probability one. This argument shows that in any iteration starting with $X_t > L$, regardless of what happened in previous iterations, we have a probability of at least $p$ of reducing $X_t$ by at least one. 

Now we analyze the probability of the event $A$ that $X_t$ drops to $L$ or less in $T = \lceil 2 (1/p) n \rceil$ generations. We consider a random variable $Z=\sum_{i=1}^{T} Z_i$, where $Z_1,\dots,Z_{T}$ are independent Bernoulli random variables with success probability of~$p$. Then $X_{t_0}-X_{t_0+T} = \sum_{t=t_0}^{T-1} (X_{t} - X_{t+1})$ stochastically dominates $Z':=\min\{Z,X_{t_0}-L\}$. Since $E[Z]\ge2n$, applying the multiplicative Chernoff bound (see, e.g.,~\cite[Theorem~1.10.7]{Doerr20bookchapter}), we have 
\begin{align*}
\Pr\left[Z \le X_{t_0}-L\right] \le \Pr\left[Z \le n-L\right] \le \Pr[Z \le n] \le \exp(-n/4).
\end{align*}
Hence 
\begin{align*}
\Pr[A] \ge \Pr\left[X_{t_0}-X_{t_0+T} \ge X_{t_0}-L \right] \ge \Pr\left[Z' \ge X_{t_0}-L\right] \ge 1-\exp(-n/4).
\end{align*}

Note that the above discussed $X_t$ is corresponding to one given $i\in[0..n-1]$. A union bound over all $i \in [0..n-1]$ gives that for any generation after the $(t_0+T)$-th generation, with probability at least 
\begin{align*}
1-n\exp(-n/4)
\end{align*}
all empty intervals in the population will have the length at most $L$, that is, we have $\mei(P_t,f_1)\le L$ for all $t \ge t_0 + T$. This proves the claimed bound with high probability.

For the claimed bound on the expectation, we note that we can repeat our argument above since we did not make any particular assumptions on the initial state. Consequently, the probability that $\mei(P_{t_0+\lambda T},f_1) > L$ is at most $(n \exp(-n/4))^\lambda$. This immediately implies that that the expected time to have all empty intervals of length at most $L$ is at most $t_0 + O(1) T = t_0 + O(n)$, see, e.g., Corollary~1.6.2 in~\cite{Doerr20bookchapter}.
\end{proof}

It remains to analyze the time it takes to have the two extreme points $0^n$ and $1^n$ in the population.
This analysis is very similar to the analysis of the original \NSGA with population size large enough that the full Pareto front is found in~\cite[Theorems~2 and~6]{ZhengLD22}. 
Since the proof below also applies to the original \NSGA, we formulate the following result for both algorithms.

\begin{lemma}
Consider using the \NSGA in the classic version or using the current crowding distance  (Algorithm~\ref{alg:nsgaii} or~\ref{alg:onthefly}) with one of the following six ways to generate the offspring, namely, applying fair selection, random selection, or binary tournament selection and applying one-bit mutation or standard bit-wise mutation. Then after an expected number of $O(n\log n)$ iterations, that is, an expected number of $O(Nn\log n)$ fitness evaluations, the two extreme points $0^n$ and $1^n$ are contained in the population.
\label{lem:twoextreme}
\end{lemma}

\begin{proof}
We first consider the time to generate $1^n$, that is, the unique search point with largest $f_1$ value. Let $t\ge 0$.  Let $x_{\max}$ be an individual in the parent population $P_t$ with largest $f_1$ value and with infinite crowding distance. Let $k=f_1(x_{\max})$. Let $p_s^k$ denote the probability that $x_{\max}$ appears at least once in the $N$ individuals selected to generate offspring and let $p_+^k$ denote the probability of generating an individual with larger $f_1$ value from $x_{\max}$ via the mutation operator. Note that if there exists individuals with $f_1$ value larger than $f_1(x)$, then one such individual will survive to $P_{t+1}$, and $P_{t+1}$ will have an individual with $f_1$ fitness at least $f_1(x_{\max})+1$. 
The expected number of iterations until this happens is at most $1/(p_s^kp_+^k)$. It is not difficult to see that the largest $f_1$ value in $P_t$ cannot decrease. Since $f_1(1^n)=n$, the expected number of iterations to reach $1^n$ is at most
$\sum_{k=0}^{n-1} \frac1{p_s^kp_+^k}$.

It is not difficult to see that $p_s^k=1$ for the fair selection and at least $1-\left(1-1/N\right)^N\ge 1-1/e$ for random selection. For binary tournament selection, there are at most two individuals with infinite crowding distance but $f_1$ value different from $f_1(x_{\max})$. Hence $p_s^k\ge 1-\left(1-\frac1N \frac{N-2}{N}\right)^N=1-\exp(-(N-2)/N)=\Theta(1)$ in this case. It is also not difficult to see that $p_+^k=(n-k)/n$ for the one-bit mutation, and $p_+^k\ge\frac{n-k}{n}(1-\frac1n)^{n-1}\ge \frac{n-k}{en}$. Hence the expected number of iterations to find $1^n$ is at most
\begin{align*}
\sum_{k=0}^{n-1} \frac1{p_s^kp_+^k}=O\left(\sum_{k=0}^{n-1} \frac{n}{n-k}\right)=O(n\log n).
\end{align*}

Similarly, we could show that the expected number of iterations to have $0^n$ in the population is also $O(n\log n)$. Hence, the expected number of iteration to have the two extreme points in the population is $O(n\log n)$. Since each iteration uses $N$ fitness evaluations, the expected number of fitness evaluations is $O(Nn\log n)$.
\end{proof}

From Lemmas~\ref{lem:approx} and~\ref{lem:twoextreme}, we easily obtain the following theorem on the approximation ability of the \NSGA using the current crowding distance.
\begin{theorem}
Let $N\ge 4$.
Consider using the \NSGA with fair or random parent selection, 
with survival selection using the current crowding distance, and one-bit mutation to optimize the \omm function with problem size $n$. Then after an expected number of $O(N n\log n)$ fitness evaluations, a population containing the two extreme points $0^n$ and $1^n$ and with $\mei(P_t,f_1)\le\max\{\frac{2n}{N-3},1\}$ is reached and kept for all future time.
\label{thm:approx}
\end{theorem}

Recalling from Lemma~\ref{lem:optimalmei} that the optimal maximal empty interval size is $\frac{n}{N-1}$, Theorem~\ref{thm:approx} shows that the gaps in the Pareto front are at most by around a factor of $2$ larger than this theoretical minimum. Also, comparing the runtimes in Lemma~\ref{lem:approx} with those in Lemma~\ref{lem:twoextreme} or Theorem~\ref{thm:approx}, we see that the cost for reaching a good approximation is asymptotically negligible compared to the one proven for reaching the two extreme points.

\section{Steady-State Approaches}\label{sec:ssnsga}

In order to overcome the shortcoming of the original \NSGA, we discussed in the previous section the approach to update the crowding distance after each removal. A different way to resolve the problems arising from in parallel removing individuals based on the crowding distance is a steady-state approach, that is, a version of the \NSGA which in each iteration generates only one new solution and selects the new population from the old one plus the new individual. Building on first ideas in this direction in~\cite{SrinivasanR06}, such a steady-state \NSGA was proposed in~\cite{DurilloNLA09} and empirically a good performance was shown.

In this section, we will theoretically prove that the steady-state \NSGA can approximate well the Pareto front. This is also the first mathematical runtime analysis of the steady-state \NSGA.

\subsection{The Steady-State \NSGA}\label{subsec:intoss}

We now make precise the steady-state \NSGA we regard in this work. Apart from the way the offspring is generated, it is identical to the steady-state \NSGA proposed in~\cite{DurilloNLA09}. In this \NSGA, with population size $N$, a single offspring is generated per iteration. Note that when generating a single offspring, fair selection is not well-defined, so we only regard random parent selection. Hence the single offspring will be generated from mutating a randomly selected parent, either via one-bit mutation or via bit-wise mutation. From the combined parent and offspring population $R$ of size $N+1$, the next parent population of size $N$ is selected in the same way as in the classic \NSGA (apart from the different population sizes), that is, non-dominated sorting is applied to $R$ and then all individuals are taken into the next generation apart from one individual in the last front with smallest crowding distance (ties broken randomly). 

We note that the above description of the steady-stage \NSGA requires to run the non-dominated sorting and crowding-distance computation procedures once for each newly generated individual, whereas in the classic \NSGA, these procedures are called only once per $N$ newly generated offspring. A more efficient implementation of the steady-state NSGA was proposed in~\cite{NigmatullinBS16}.
Readers can refer to this paper for more details.

\subsection{Runtime Analysis and Approximation Quality of the Steady-State \NSGA}\label{subsec:runtimess}

We now prove runtime guarantees and approximation guarantees for the steady-state \NSGA of the same quality as those proven for the \NSGA using the current crowding distance in Section~\ref{subsec:approx}. 

We first note that the removal of a least favorable individual from $N+1$ individuals, which is the selection step of the steady-state \NSGA, was already analyzed as a special case of Lemma~\ref{lem:selection}, namely as the last of the $N$ removal steps in the \NSGA with current crowding distance. Since there no particular structure of the $N+1$ individuals was assumed, the proof of this lemma immediately implies the following result

\begin{lemma}\label{lem:selectionss}
Let $N\ge 4$.
Consider using the steady-state \NSGA to optimize the \omm function with problem size~$n$. Let $t_0$ be the first generation such that the two extreme points $0^n$ and $1^n$ are in $P_{t_0}$. Let $t\ge t_0$. Consider the selection of the next population $P_{t+1}$ from $R_t$. Let $x \in R_t$ with crowding distance $\frac{4}{N-3}$ or higher. Then $x \in P_{t+1}$. 

Also, $P_{t+1}$ surely contains the two extreme points.
\end{lemma}
Fix some $i \in [0..n-1]$. We reuse from Section~\ref{subsec:approx} the notations of $X_t$ being the length of the empty interval in $f_1(P_t)$ that contains $i+0.5$, and $Y_t$ being the corresponding length for $R_t$. 

From Lemma~\ref{lem:selectionss}, in a fashion analogous to the proof of Lemma~\ref{lem:approx}, we obtain first that if for some $t\ge t_0$ and $M \ge L:=\max\{\frac{2n}{N-3},1\}$ we have $Y_t\le M$, then $X_{t'} \le M$ for all $t'>t$. 
Consequently, (i)~once $X_t\le L$, we have $X_{t'}\le L$ for all $t' \ge t$, and (ii)~if $X_t\ge L$, then $X_{t+1} \le X_t$.

Analogous to the second part of the proof in Lemma~\ref{lem:approx}, we have for the steady-state \NSGA that with probability $\Omega(1/N)$, the empty-interval length $X_t$ reduces in one iteration, that is, that $X_{t+1} \le X_t -1$, when $X_t > L$. This probability was at least $\frac{1}{2e}(1-\frac1e)$  in Lemma~\ref{lem:approx} where we consider the event that a desired individual is contained 
in the $N$ selected parents. Now where only one parent is selected, it is smaller by a factor of $\Theta(N)$. These arguments, as in the proof of Lemma~\ref{lem:approx}, give the following lemma.

\begin{lemma}
Let $N\ge 4$.
Consider using the steady-state \NSGA with random parent selection and with one-bit or bit-wise mutation to optimize the \omm function with problem size $n$. Let $t_0$ be the first generation that the two extreme points $0^n$ and $1^n$ are contained in $P_{t_0}$. Let $t_1 \ge t_0$ be the first generation such that $\mei(P_{t_1},f_1)\le \max\{\frac{2n}{N-3},1\} =: L$. Then $t_1 - t_0 = O(Nn)$, both in expectation and with probability $1-o(1)$. Also, for all $t > t_1$, we have $\mei(P_t,f_1) \le L$ with probability one.
\label{lem:approxss}
\end{lemma}

It remains to analyze the time to have the two extreme points $0^n$ and $1^n$ in the population. This analysis is completely identical to the one in Lemma~\ref{lem:twoextreme} apart from the calculation of the probability $p^k_s$, which now is the probability that the unique parent selected in one iteration is an individual with maximal $f_1$-value and infinite crowding distance. 

For random selection, this probability is at least $p^k_s \ge 1/N$. For binary tournament selection, $p^k_s \ge \frac 1N \frac{N-2}{N} = \frac{N-2}{N^2}$, again using the argument that there are at most two individuals with infinite crowding distance and $f_1$-value different from the maximal value. Naturally, these probabilities are by a factor of $\Theta(N)$ smaller than the values computed in Lemma~\ref{lem:twoextreme}, where in each iteration $N$ offspring were generated. Consequently, in terms of iterations, our bound for the time to find the two extremal points is by a factor of $\Theta(N)$ larger.
\begin{lemma}
Consider using the steady-state \NSGA with one of the following four ways to generate the offspring, namely, applying random selection or binary tournament selection and applying one-bit mutation or standard bit-wise mutation. Then after an expected number of $O(Nn\log n)$ iterations (fitness evaluations), the two extreme points $0^n$ and $1^n$ are contained in the population.
\label{lem:twoextremess}
\end{lemma}

From Lemmas~\ref{lem:approxss} and~\ref{lem:twoextremess}, we have the following theorem for the approximation ability of the steady-state \NSGA.
\begin{theorem}
Let $N\ge 4$.
Consider using the steady-state \NSGA with random selection and with one-bit or standard bit-wise mutation to optimize the \omm function with problem size $n$. Then after an expected number of $O(N n\log n)$ fitness evaluations, a population containing the two extreme points $0^n$ and $1^n$ and with $\mei(P_t,f_1)\le\max\{\frac{2n}{N-3},1\}$ is reached and kept for all future time.
\label{thm:approxss}
\end{theorem}

Note that except Lemma~\ref{lem:twoextremess} also discussing the binary tournament selection, we only consider random parent selection in other propositions in this section. It is not clear how to extend them to tournament selection, and we leave it as an interesting open question for future research.

\section{Experiments}\label{sec:exp}

In Section~\ref{sec:nsgaii} we conducted a theoretical analysis of two synthetic situations to show that the traditional \NSGA can have difficulties in approximating the Pareto front. The complicated population dynamics prevented us from analyzing how often such situations arise. 
In Section~\ref{sec:onthefly} and~\ref{sec:ssnsga}, we proved that the \NSGA using the current crowding distance and the steady-state \NSGA leave gaps on the Pareto front that are asymptotically at most a factor of $2$ larger than those given by a perfect approximation of the Pareto front. To complete the picture, in this section, we present some experimental results.

\subsection{Settings}
We conduct experiments with the following settings.
\begin{itemize}
\item Problem: \omm, the benchmark in our theoretical results in Sections~\ref{sec:nsgaii} and~\ref{sec:onthefly}. 
\item Problem size $n$: 601. Given that the \oneminmax problem is an easy multi-objective problem, this is a moderate problem size. Such a choice is sensible, since with a too small size we will not gain reliable insights, whereas insights obtained on a large problem size raise the question whether they still apply to practically relevant problem sizes. We use the odd number $n=601$ to include the setting discussed in Lemma~\ref{lem:bigMEI}.
\item Algorithms: We regard the classic \NSGA using the initial crowding distance for the selection, the \NSGA using the current crowding distance, and the steady-state \NSGA. As in our theoretical analysis, we do not use crossover. That is, mutation is the only operator to generate the offspring population. 
\item Mating selection and mutation strategy: We select the parents of the mutation operations via fair selection for the two generational \NSGA variant. For the steady-state \NSGA, we select a random individual as parent. As mutation operators, we use one-bit mutation. These setting are included in our mathematical runtime analysis. We have no reason to believe that other settings, e.g., random selection and bit-wise mutation (as also covered in our mathematical analyses), lead to substantially different results.
\item Population size $N$: $(n+1)/2=301, \lceil (n+1)/4 \rceil=151,$ and $\lceil (n+1)/8\rceil=76$. We choose $(n+1)/2$ as this value is used in Lemma~\ref{lem:bigMEI}. We did not regard larger population sizes, since for $N=n+1$ experiments where conducted in~\cite{ZhengLD22}. The two smaller population sizes are used to see how the approximation ability scales with the population size.
\item $20$ independent runs for each setting.
\end{itemize}

\subsection{Results}

Our focus is the maximal length of an empty interval on the Pareto front, that is, the $\mei$ value defined earlier. As long as the population has not fully spread out on the Pareto front, that is, the extremal solutions $0^n$ and $1^n$ are not yet part of the population, enlarging the spread of the population is more critical than a balanced distribution in the known part of the front. For this reason, we only regard times after both extremal solutions have entered the population. 

To see whether the approximation quality changes over time, we regard separately the two intervals of $[1..100]$ and $[3001..3100]$ generations after finding the two extremal solutions. We collect statistical data on the $\mei$ value in these intervals in Table~\ref{tbl:omm} and we display exemplary runs in Figure~\ref{fig:omm}.
\begin{table}[h!]
\centering
\caption{The first, second, and third quartiles (displayed in the format of $(\cdot,\cdot,\cdot)$) for the maximal empty interval sizes $\mei$ within $100$ generations and $20$ independent runs. We regard separately generations $[1..100]$ and $[3001..3100]$ after the two extremal points have entered the population. For the steady-state \NSGA with population size $N$, we regard generations $[1..100N]$ and $[3000N+1..3100N]$ for a fair comparison.}
\label{tbl:omm}
\begin{tabular}{lcc}
\toprule
Generations & $[1..100]$ & $[3001..3100]$ \\
\midrule
\multicolumn{3}{l}{\underline{$N=301$}} \\
Initial CD & (7,8,9) & (7,8,9)  \\
Current CD & (3,3,3) & (3,3,3) \\
Steady-State & (3,3,3) & (3,3,3) \\
\midrule
\multicolumn{3}{l}{\underline{$N=151$}} \\
Initial CD & (14,15,17) & (14,15,17)  \\
Current CD & (5,5,6) & (5,5,6)  \\
Steady-State & (5,5,5) & (5,5,5) \\
\midrule
\multicolumn{3}{l}{\underline{$N=76$}} \\
Initial CD & (23,26,29) & (24,27,30)  \\
Current CD & (11,12,12) & (11,12,12)  \\
Steady-State & (11,12,12) & (11,11,11) \\
\bottomrule
\end{tabular}
\end{table}

\begin{figure}
\centering
\includegraphics[width=0.75\columnwidth]{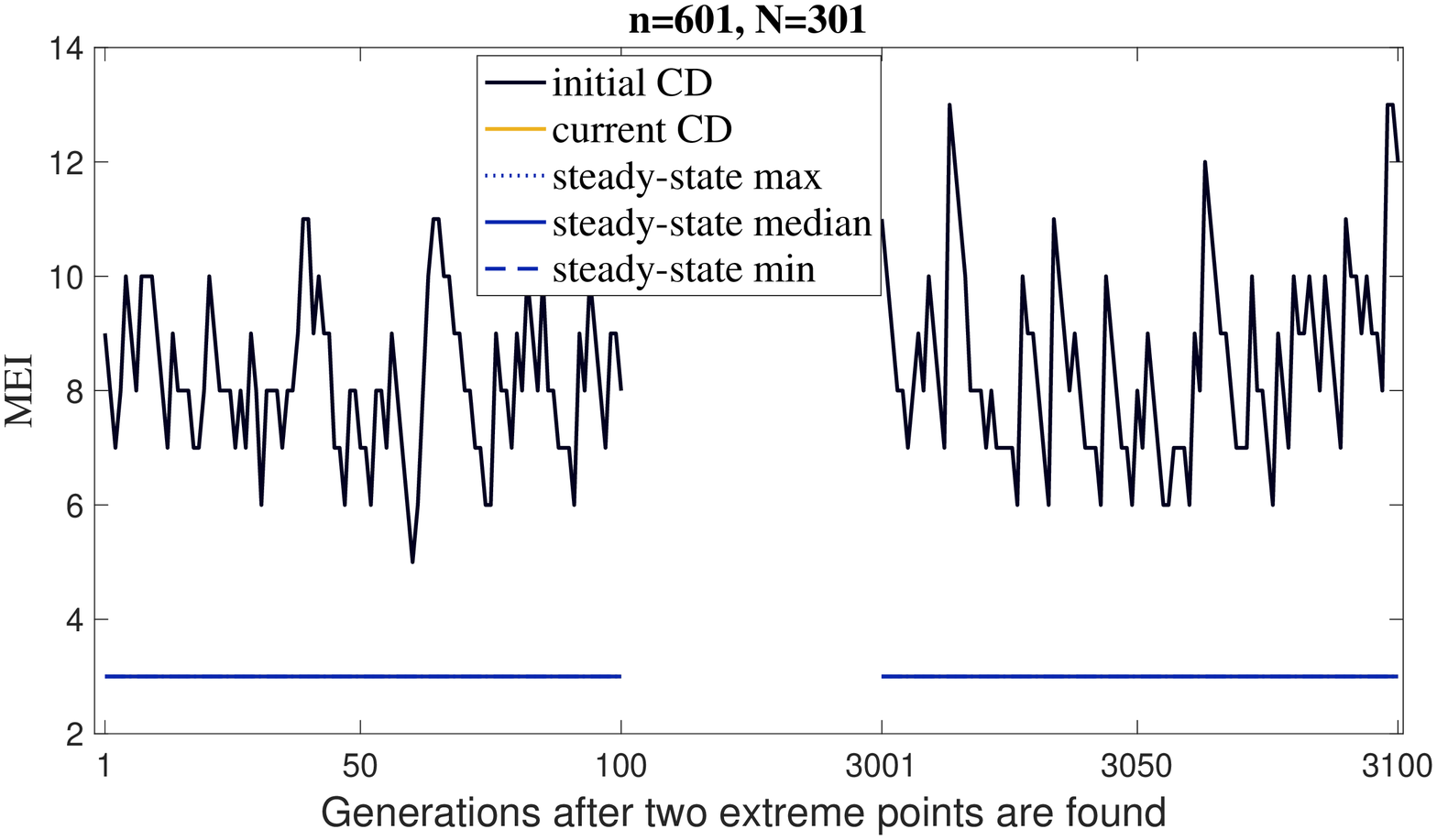} 
\includegraphics[width=0.75\columnwidth]{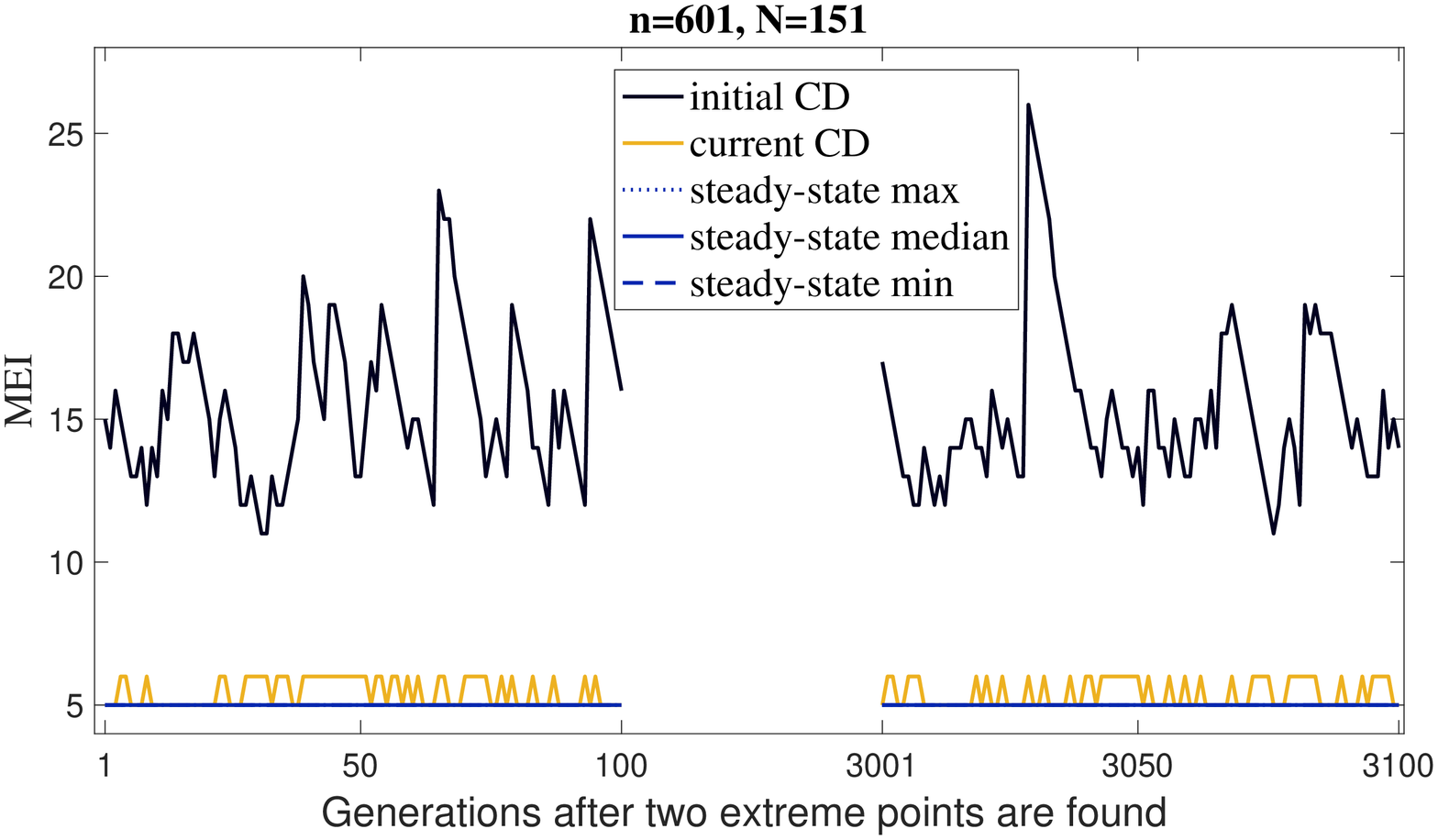}
\includegraphics[width=0.75\columnwidth]{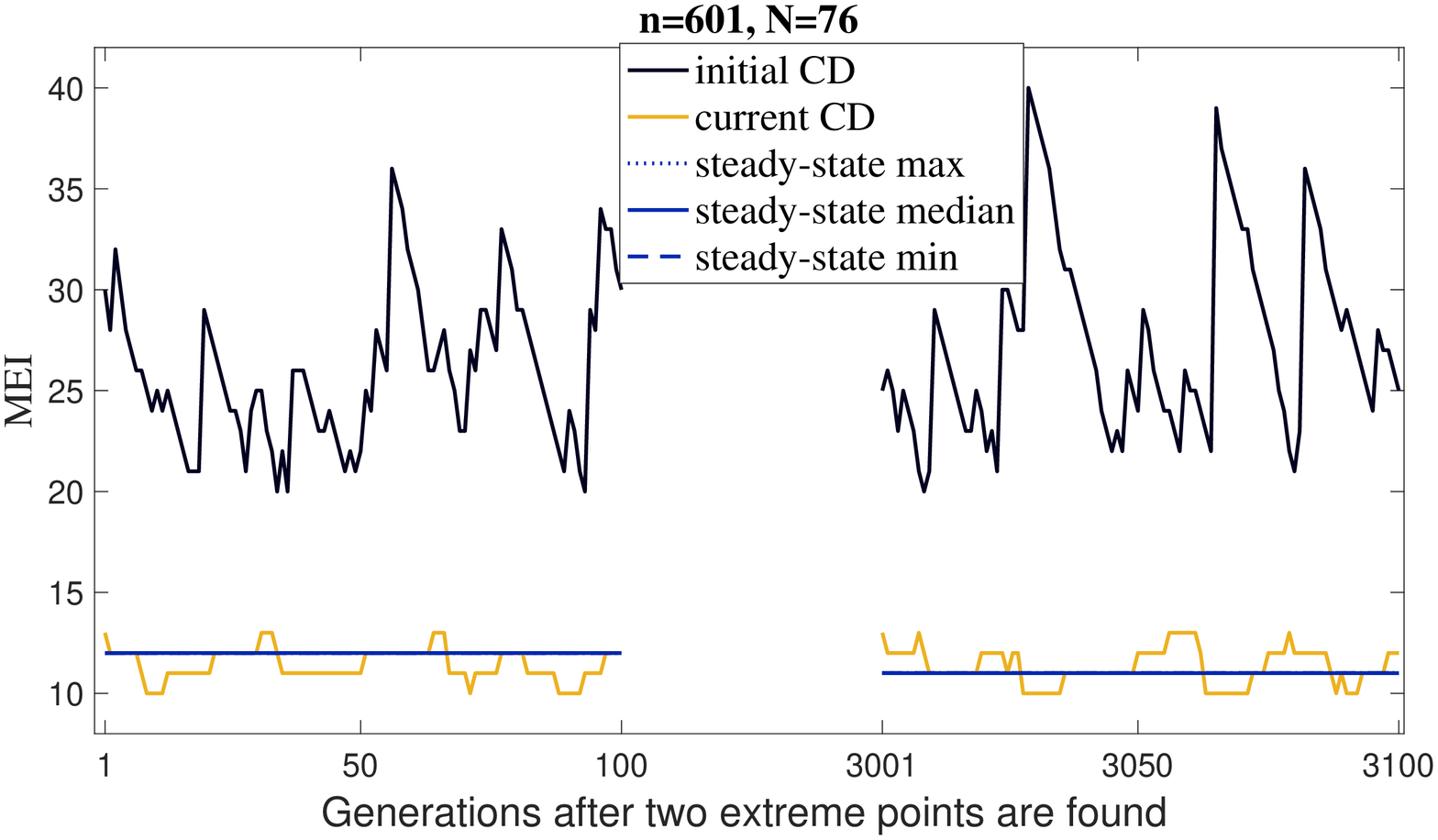} 
\caption{The $\mei$ for generations $[1..100]$ and $[3001..3100]$ after the two extreme points were found, in one exemplary run. For the steady-state \NSGA with population size $N$, at point $i$ on the x-axis actually the data for generations $[1+N(i-1)..Ni$ are displayed, namely the maximal, minimal, and median MEI value in this interval (which always happened to be identical, and also identical to the current-CD value for $N=301$). }
\label{fig:omm}
\end{figure}

The optimal MEI value $\lceil n/(N-1) \rceil$ for \omm with $n=601$ and population sizes $N=301,151,76$ equals $3,5,$ and $9$, respectively. From Table~\ref{tbl:omm}, we see that the modified \NSGA and the steady-state \NSGA often reach the optimal MEI for $N=301$ and $151$, and are only slightly sub-optimal values for $N=76$. In contrast, the traditional \NSGA shows median MEI values of $8,15,26$ in the better time interval. This is more than twice the optimal MEI and the median MEI of the two \NSGA variants that overcome the problem with the initial crowding distance. 

We observe no greater differences between the $100$ generations right after finding the two extremal points and the $100$ generations $3000$ generations later. For $N=76$, the \NSGA with initial crowding distance displays slightly larger MEI values in the later interval, whereas the steady-state \NSGA shows slightly smaller values. We do not have an explanation for this, but the small differences render it not very interesting to investigate this observation further.

In Figure~\ref{fig:omm}, we see that the MEI value oscillates considerably for the classic \NSGA, whereas it is relatively stable for the \NSGA using the current crowding distance and constant for the steady state \NSGA. This appears to be the second advantage of the two variants of the \NSGA.

Our experimental data is not sufficient to answer the question if the traditional \NSGA suffers from super-constant MEI values. Our theoretical result in Lemma~\ref{lem:bigMEI} could be seen as an indication that logarithmic MEI values can show up (or MEI values of order $\Theta(\frac{n}{N}\log N)$ for general values of $N\le n$). To answer this question, significantly more experiments with truly large problem sizes would be necessary (due to the slow growth behavior of logarithmic functions). For our purposes, our results, however, are fully sufficient. They show clearly that the two variants of the \NSGA not building on the initial crowding distance yield much smaller and more stable MEI values. Not surprisingly, the experimentally observed MEI values for these two variants are  better than the mathematical guarantee given in Theorems~\ref{thm:approx} and~\ref{thm:approxss}. 


\section{Conclusion}\label{sec:con}

None of the existing runtime analyses for the \NSGA (including those published after our conference version~\cite{ZhengD22gecco}) regards the performance of the \NSGA when the population size is smaller than the size of the Pareto front. In this work, we regard this situation and discuss how well the population evolved by the \NSGA approximates the Pareto front. Our theoretical analysis of two artificial cases and our experiments give a mixed picture. However, they also suggest that the reason for the not fully satisfying approximation behavior is the fact that the selection of the next parent population is based on the initial crowding distance of individuals in the combined parent and offspring population, which can be very different from the crowding distance at the moment when an individual is removed, an effect previously observed with less mathematical arguments in~\cite{KukkonenD06}.

For the \NSGA building on the current crowding distance, we proved very positive approximation results for the \oneminmax benchmark. After an expected time comparable to the runtime of the classic \NSGA on the \oneminmax benchmark, a population is reached that covers the Pareto front apart from gaps that are only a constant factor larger than in an optimal approximation. This state of the population is maintained for the remaining runtime, with probability one. We further discussed the steady-state \NSGA and proved the same good approximation ability. Our experiments confirm the superiority of these two selection strategies.

From our proofs, we conjecture that similar results can be obtained for other classic benchmark problems such as \lotz or the large front problem~\cite{HorobaN08}. 
Overall, this work gives additional motivation to prefer two less common variants of the \NSGA, the \NSGA working with the current crowding distance~\cite{KukkonenD06} (around 180 citations on Google scholar) and the steady-state \NSGA~\cite{DurilloNLA09} (76 citations), over the classic \NSGA (over 50,000 citations).

We want to mention two open problems arising from this work. As already discussed, due to the complex population dynamics, we were not able to conduct a mathematical analysis of the approximation ability of the classic \NSGA. So neither we could prove, say, a super-constant lower bound on the typical MEI value for, say, $N \approx n/2$, nor we could prove any interesting upper bound on the MEI value, say a logarithmic factor above the ideal value. This is clearly an interesting problem area to further explore. On the more technical side, we note that in this first work on the approximation strength of the \NSGA we only regarded fair and random parent selection. Our proofs do not apply to tournament selection, and we feel that substantially different arguments will be necessary to prove approximation guarantees in this case. 

\section*{Acknowlegements}
This work was supported by National Natural Science Foundation of China (Grant No. 62306086), Science, Technology and Innovation Commission of Shenzhen Municipality (Grant No. GXWD20220818191018001), Guangdong Basic and Applied Basic Research Foundation (Grant No. 2019A1515110177).

This work was also supported by a public grant as part of the Investissement d'avenir project, reference ANR-11-LABX-0056-LMH, LabEx LMH.

We thank Hisao Ishibuchi for pointing out Kukkonen and Deb's work~\cite{KukkonenD06}. We also thank Ke Shang for his question about the behavior of the steady-state \NSGA at GECCO 2022.


\newcommand{\etalchar}[1]{$^{#1}$}

\end{document}